\pdfoutput=1
\documentclass[11pt]{article}
\PassOptionsToPackage{numbers}{natbib}
\PassOptionsToPackage{colorlinks=true,citecolor=blue,urlcolor=blue,linkcolor=blue}{hyperref}
\usepackage[preprint]{styles/neurips_2026}
\usepackage{xurl}
\usepackage{amsthm}
\usepackage{mathtools}
\usepackage{amssymb}
\usepackage{enumitem}
\usepackage{xcolor}
\usepackage{tikz}
\usepackage{pgfplots}
\usepackage{thmtools}
\usepackage{float}
\usepackage{booktabs}
\usepackage{tabularx}
\usepackage{graphicx}
\usepackage{orcidlink}
\usepackage{algorithm}
\usepackage{algpseudocode}
\usetikzlibrary{shapes,arrows,positioning,calc,decorations.pathmorphing}
\pgfplotsset{compat=1.18}

\theoremstyle{plain}
\newtheorem{theorem}{Theorem}[section]
\newtheorem{lemma}[theorem]{Lemma}
\newtheorem{proposition}[theorem]{Proposition}
\newtheorem{corollary}[theorem]{Corollary}
\theoremstyle{definition}
\newtheorem{definition}[theorem]{Definition}

\theoremstyle{remark}
\newtheorem{remark}[theorem]{Remark}
\newtheorem{conjecture}[theorem]{Conjecture}

\newenvironment{proofsketch}{\noindent\emph{Proof sketch.}}{\hfill$\square$\medskip}

\DeclareMathOperator{\TV}{TV}

\DeclareMathOperator*{\argmin}{arg\,min}
\DeclareMathOperator*{\argmax}{arg\,max}
\newcommand{\cW}{\mathcal{W}}

\newcommand{\cF}{\mathcal{F}}

\newcommand{\cH}{\mathcal{H}}
\newcommand{\bR}{\mathbb{R}}
\newcommand{\bE}{\mathbb{E}}
\newcommand{\bP}{\mathbb{P}}
\newcommand{\Piop}{\Pi^{\mathrm{op}}_{m,T}}
\newcommand{\Piclk}{\Pi^{\mathrm{clk}}_{m,T}}
\newcommand{\Qop}{Q^{\mathrm{op}}_{m,T}}
\newcommand{\Qclk}{Q^{\mathrm{clk}}_{m,T}}
\newcommand{\Gobs}{\mathcal{G}_{\mathrm{obs}}}

\title{The Myhill--Nerode Theorem for Bounded Interaction:\\
Canonical Abstractions via Agent-Bounded Indistinguishability}

\author{
Anthony T Nixon\thanks{
DefSig. Correspondence: \texttt{anthony@defsig.com}. \orcidlink{0000-0002-0289-0954}
}}

\date{}

\begin{document}

\maketitle

\begin{abstract}
Any capacity-limited observer induces a canonical quotient on its environment: two situations that no bounded agent can distinguish are, for that agent, the same.
We formalise this principle for finite POMDPs.
A fixed probe family of finite-state controllers induces a closed-loop Wasserstein pseudometric on observation histories and a probe-exact quotient that merges histories no controller in the family can distinguish.
The resulting quotient is canonical, minimal, and unique---a bounded-interaction analogue of the Myhill--Nerode theorem.
For clock-aware probes, the quotient is exactly decision-sufficient for agent-accessible objectives measurable with respect to the joint observation-action trajectory; for latent-state rewards, we instead rely on the observation-Lipschitz approximation bound.
The framework also suggests a reusable discrete template for a few neighboring finite or discretised settings; we record those illustrative analogies in the appendix rather than treating them as central contributions of the paper.
The exact theorem object is the clock-aware quotient; the scalable deterministic-stationary experiments study a tractable coarsening whose gap is quantified on small exact cases and treated empirically at larger scale.
We validate the theorem-level claims on Tiger and GridWorld and report operational case studies of the tractable coarsening on Tiger, GridWorld, and RockSample; these uncertified case studies are included as exploratory diagnostics of approximation behaviour and runtime, not as part of the paper's theorem-facing claim set once the exact cross-family certificate is unavailable. Heavier operational stress tests are archived separately in the appendix and artifact package.
\end{abstract}

\section{Introduction}
\label{sec:intro}

When are two environments ``the same''?  The answer depends on who is asking.
Two situations that no bounded observer can distinguish are, \emph{for that observer}, the same; the right abstraction is therefore matched to the observer's capacity.
This paper makes this principle precise for the finite discrete case and shows that the resulting quotient is canonical, minimal, and unique---a bounded-interaction analogue of the Myhill--Nerode theorem from formal language theory. We use the word \emph{analogue} deliberately: the classical suffix-test theorem is the motivating template, but our object lives on histories under closed-loop bounded control rather than as a literal restatement of the DFA setting. The analogy is structural---canonical minimal quotient under a finite equivalence relation---rather than algorithmic: unlike the classical DFA result, we do not claim a polynomial-time exact construction algorithm, and exact computation and scalable approximation are handled separately through the operational toolkit developed later in the paper.

Planning under partial observability is hard because both the history tree and the belief process explode with horizon. A natural response is to ask a bounded question: which distinctions in the environment matter to a bounded agent? If two histories induce the same closed-loop future behavior for every controller in a bounded probe class, then no controller in that class can exploit the distinction. The right abstraction target is therefore not the full belief space, but the quotient induced by bounded-agent indistinguishability.
In classical Myhill--Nerode terms, two prefixes are merged when no admissible future continuation can separate them; here the continuation is closed-loop and restricted to what a bounded controller can actually do.

This point of view yields a canonical-form result. For a fixed probe family, the \emph{probe-exact} quotient (exact for the stated probe family) is the unique minimal abstraction preserving all future observation laws visible to that family. We emphasize the use case that makes this object compelling: for objectives measurable with respect to what the bounded agent actually sees and does, the clock-aware exact quotient is decision-sufficient, not merely compressive. For latent-state rewards, the theory falls back to the approximate value-loss bound rather than exact sufficiency, and we make that limitation explicit in the abstract, experiments, and discussion.

This paper also makes a clean separation between theory and computation. The theorem-level results are stated for the clock-aware controller class $\Piclk$, where stage indexing prevents parameter reuse across time and therefore restores deterministic witness sufficiency. The experiments, by contrast, use the tractable operational family $\Piop = \Pi^{\mathrm{det,stat}}_{m,T}$ unless stated otherwise. The operational quotient $\Qop$ is a coarsening of the theorem object $\Qclk$; Section~\ref{sec:experiments} first quantifies that gap on small clock-aware exact cases and then reports operational case studies of $\Qop$ at larger scales. When that cross-family gap is measured on a tractable exact case, Theorem~\ref{thm:cross-family-transfer} turns it into an additive value-transfer certificate. When no such certificate is available, the larger tables are included only as empirical evidence about the tractable surrogate $\Qop$, not as direct empirical validation of the full clock-aware theorem object.

This paper makes the following contributions:
\begin{enumerate}[label=(\roman*),leftmargin=2em,itemsep=1pt]
\item A probe-family-parametrized closed-loop pseudometric on finite POMDPs, together with an observation-Lipschitz value-loss bound for partition-based approximate quotients (\S\ref{sec:pseudometric}; Theorem~\ref{thm:value-bound}).
\item A bounded-interaction Myhill--Nerode theorem: for a fixed probe family, the probe-exact quotient is canonical, minimal, and unique up to isomorphism (\S\ref{sec:quotient}; Theorem~\ref{thm:myhill-nerode}).
\item A clock-aware exact sufficiency theorem: for every objective in $\Gobs$ measurable with respect to the joint observation-action trajectory, the clock-aware quotient $\Qclk$ preserves expected return exactly and therefore preserves optimal value over $\Piclk$.
\item A witness boundary: deterministic sufficiency holds for clock-aware bounded controllers, while deterministic stationary controllers are shown explicitly to be insufficient in general.
\item A formal bridge from theorem to computation on \emph{measured exact cases}: if one probe family is contained in another, its quotient is a coarsening, and a measured cross-family probe gap yields an additive value-transfer bound on that same bounded regime. In particular, $\Qop$ is a tractable coarsening of $\Qclk$, with value loss controlled by $\varepsilon + \delta_{\mathrm{clk}}$ only when the clock-aware/operational gap $\delta_{\mathrm{clk}}$ is explicitly measured for the benchmark and horizon under discussion.
\item A companion operational toolkit built from controller-subset certificates, sampling-based probe estimation, layered horizon decomposition, and observation coarsening, together with exact small-case alignment for $\Qclk$ (Section~\ref{sec:decision-sufficiency}; Table~\ref{tab:observation-planning}) and exact-for-family deterministic-stationary certificates on small benchmarks (Table~\ref{tab:partition-agreement}). These certified strata participate in the paper's claim-evidence chain. Additional medium-scale case studies on GridWorld, random POMDPs, and RockSample are reported later as exploratory diagnostics of the tractable surrogate $\Qop$, not as further theorem-facing contribution claims. Heavier scale-up and long-horizon stress tests are reported separately as appendix-only archival evidence.
\end{enumerate}

Accordingly, whenever we use the phrase \emph{exact decision sufficiency} below, it refers only to objectives in $\Gobs$ measurable on the joint observation-action trajectory; latent-state rewards are covered only by the approximate observation-Lipschitz bound.
Likewise, whenever we invoke the cross-family transfer guarantee below, it refers only to the tractable exact cases where the corresponding $\delta_{\mathrm{clk}}$ is explicitly reported; no larger operational scaling row is claimed to inherit that certificate unless the matching gap measurement is shown for that same bounded regime.
The paper's theorem-facing evidence set is therefore exhausted by the theorem statements themselves, the clock-aware exact tables, and the deterministic-stationary rows carrying an explicit subset or cross-family certificate; larger uncertified operational tables are included only as exploratory diagnostics of $\Qop$.

At the witness boundary, deterministic sufficiency holds for the clock-aware FSC class of Theorem~\ref{thm:witness}, but Proposition~\ref{prop:stationary-counterexample} shows that this cannot be extended to stationary looping FSCs in general. The framework's parametrisation by $(m, T, \delta_O)$ makes the observer-capacity mismatch explicit without requiring a commitment to any broader operational reduction outside the finite POMDP setting.

\paragraph{Proof roadmap.}
The main proof structure is short. The pseudometric theorem is the Wasserstein triangle inequality plus maximisation over bounded closed-loop controllers. The value bound is a per-stage observation-Lipschitz transfer summed over horizon. The quotient theorem then combines right-invariance of history equivalence, induction on preserved observation laws, and a universal-object argument for minimality and uniqueness. The same ingredients are reused in the approximate and layered variants later in the paper.

\paragraph{Related work.}
\textit{State abstraction and bisimulation.}
State abstraction in MDPs has a rich theory: Li et al.~\citep{li2006towards} provide a taxonomy, Abel et al.~\citep{abel2016state} introduce agent-aware abstraction, and Abel~\citep{abel2022thesis} offers a comprehensive account grounded in category theory.
Bisimulation metrics~\citep{ferns2004metrics,ferns2011bisimulation} provide a continuous alternative to exact equivalence; Calo et al.~\citep{calo2024bisimulation} recently showed these are optimal-transport distances, and Kemertas and Aumentado-Armstrong~\citep{kemertas2022towards} extended metric learning to robust settings.
Deep bisimulation methods~\citep{gelada2019deepmdp,zhang2021learning,castro2020scalable} scale these ideas to high dimensions via learned encoders; our model-based framework provides a complementary theoretical target (Appendix~\ref{app:related}).

\textit{POMDP equivalences and predictive representations.}
For POMDPs, Castro~\citep{castro2009equivalence} formalized exact bisimulation, and Dean and Givan~\citep{dean1997model} introduced homogeneous partitions.
Two nearby traditions are worth separating. Restricting probes to open-loop action strings or constant-action FSCs recovers a notion approaching Castro-style conditional observation-sequence equivalence; probabilistic testing equivalence in the Larsen--Skou lineage \citep{larsen1991bisimulation} is close in spirit, but the tests there are external experiments on labelled stochastic processes, whereas ours are embodied as finite-memory controllers interacting with a controlled partially observable environment. The new ingredient is a closed-loop bounded-observer equivalence that is explicitly parameterized by controller capacity and materialized as a quotient POMDP.

Predictive state representations, observable operator models, and their controlled or spectral variants \citep{littman2001predictive,jaeger2000observable,boots2011closing,balle2014methods} define state through predictions of future observations, but through open-loop tests or Hankel-style predictive structure rather than closed-loop, bounded-capacity FSC probes. Our quotient is a closed-loop bounded-observer analogue of that predictive-state philosophy, with controller capacity $(m,T,\delta_O)$ replacing predictive rank as the organizing parameter. The comparison with controlled PSRs---which already incorporate actions---is structural rather than order-theoretic: PSR-style low-rank structure captures linear predictive redundancy, while our anchor-rank condition captures the stronger max-preserving redundancy needed for quotient construction (Appendix~\ref{app:structural-subset}).

\textit{Positioning.}
FSCs as policy representations are due to Poupart and Boutilier, Hansen, and Amato et al.~\citep{poupart2003bounded,hansen1998solving,amato2010optimizing}; we repurpose them as \emph{probes} determining abstraction granularity.
Our indistinguishability criterion is related in spirit to comparison-of-experiments ideas \citep{blackwell1953equivalent,torgersen1991comparison}, but the signal structure is closed-loop and policy-dependent.
Relative to exact POMDP bisimulation, the object here is history-based and explicitly parametrised by controller capacity $(m,T,\delta_O)$; relative to agent-aware abstraction, it yields a canonical minimal quotient preserving the full closed-loop observation law rather than a single value function.
Conceptually, the pseudometric is a history-space bounded-agent analogue of Wasserstein bisimulation metrics: the lineage is direct, but $D_T^\Pi$ compares finite-horizon observation-sequence laws and restricts the supremum to a prescribed bounded controller family. The novelty claim is the bounded-family quotient and its exact/approximate sufficiency consequences, not the mere use of Wasserstein distance.
Extended discussion and additional references are in Appendix~\ref{app:related}.

\begin{figure}[htbp]
\centering
\begin{tikzpicture}[
  box/.style={draw, rounded corners=3pt, minimum height=0.9cm, minimum width=2.0cm,
              font=\small, align=center, fill=blue!6},
  ubox/.style={draw, rounded corners=3pt, minimum height=0.9cm, minimum width=2.2cm,
              font=\small, align=center, fill=orange!8},
  arr/.style={->, thick, >=stealth},
  lbl/.style={font=\scriptsize, midway, above}
]
\node[font=\small\bfseries, anchor=west] at (-0.5, 1.7) {(a) Universal schema};
\node[ubox] (hidden) at (0,0.9) {Hidden\\State};
\node[ubox] (channel) at (3.2,0.9) {Observation\\Channel};
\node[ubox] (observer) at (6.8,0.9) {Bounded Observer\\$(m, T, \delta_O)$};
\node[ubox] (canonical) at (10.4,0.9) {Canonical\\Quotient};
\draw[arr] (hidden) -- (channel);
\draw[arr] (channel) -- (observer);
\draw[arr] (observer) -- (canonical);
\node[font=\small\bfseries, anchor=west] at (-0.5, -0.6) {(b) POMDP instantiation};
\node[box] (hist) at (0,-1.5) {Histories\\$O^{\le T}$};
\node[box] (fsc) at (2.8,-1.5) {FSC Probes\\$\Pi_{m,T}$};
\node[box] (dist) at (5.6,-1.5) {$\cW_1$ Distance\\Matrix};
\node[box] (clust) at (8.4,-1.5) {$\varepsilon$-Clustering};
\node[box] (quot) at (11.2,-1.5) {Quotient\\POMDP};
\draw[arr] (hist) -- (fsc);
\draw[arr] (fsc) -- (dist);
\draw[arr] (dist) -- (clust);
\draw[arr] (clust) -- (quot);
\node[font=\scriptsize, below=0.15cm of dist, text=gray] {max over $\pi$};
\node[font=\scriptsize, below=0.15cm of clust, text=gray] {complete linkage};
\node[font=\scriptsize, below=0.15cm of quot, text=gray] {Thm~\ref{thm:value-bound}: $|V^\pi - \tilde{V}^\pi| \le L_R T \varepsilon$};
\end{tikzpicture}
\caption{\textbf{(a)} Any system with hidden state, an observation channel, and a bounded observer $(m, T, \delta_O)$ admits a canonical quotient---a minimal equivalence over situations the observer cannot distinguish. \textbf{(b)} The POMDP instantiation: observation histories are probed by all bounded FSCs; pairwise $\cW_1$ distances (maximised over policies) are clustered; the quotient POMDP preserves all observation laws with certified value loss.}
\label{fig:pipeline}
\end{figure}
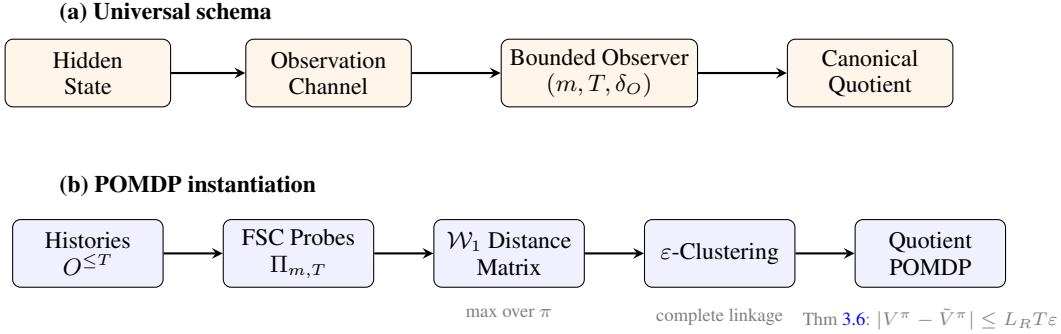

\section{Preliminaries}
\label{sec:prelim}

\begin{definition}[Finite POMDP]
\label{def:pomdp}
A finite POMDP is a tuple $M = \langle S, A, O, P, Z, R, b_0 \rangle$ where $S$ is a finite state set, $A$ a finite action set, $O$ a finite observation set, $P \colon S \times A \to \Delta(S)$ a transition kernel, $Z \colon S \times A \to \Delta(O)$ an observation kernel, $R \colon S \times A \to \bR$ a reward function, and $b_0 \in \Delta(S)$ an initial belief.
\end{definition}

\begin{definition}[Stochastic Finite-State Controller]
\label{def:fsc}
A stochastic FSC is a tuple $\pi = \langle N, \alpha, \beta, n_0 \rangle$ where $N$ is a finite set of internal nodes with $|N| \leq m$, $\alpha \colon N \to \Delta(A)$ is an action-selection function, $\beta \colon N \times O \to \Delta(N)$ is an internal-transition function, and $n_0 \in N$ is the initial node.
\end{definition}

Fix a memory bound $m$ and horizon $T$. We distinguish two probe families.
\begin{itemize}[leftmargin=1.5em,itemsep=1pt]
\item \textbf{Clock-aware probes} $\Piclk$: FSCs with stage-indexed maps $(\alpha_\tau,\beta_\tau)_{\tau=0}^{T-1}$. Equivalently, these are FSCs on node space $N \times \{0,\ldots,T-1\}$, where the external horizon clock does not count against the internal memory budget.
\item \textbf{Operational probes} $\Piop$: deterministic stationary FSCs with at most $m$ internal nodes. The fully enumerated large-scale experiments use this deterministic-stationary family as a tractable operational restriction.
\end{itemize}
Practically, a clock-aware FSC may change its action and node-update rule with the time index even if it revisits the same internal node, whereas a stationary FSC must reuse the same rule whenever that node is revisited. For example, a one-node clock-aware controller may ``listen'' at stage~1 and ``open'' at stage~2 while remaining in the same internal node; a one-node stationary controller cannot express that time-indexed switch without enlarging its state. This is why clock-awareness restores deterministic witness sufficiency in the theorem, while the stationary family is a stricter operational restriction.
We therefore use $\Piclk$ as a gold-standard bounded-observer family: it is partly proof-motivated, but it also models bounded agents whose internal logic can legitimately depend on a known stage or deadline (for example countdown-style or phase-based controllers). The stationary family $\Piop$ is the stricter operational deployment model when such external clocking is unavailable or intentionally disallowed.

For any fixed probe family $\Pi$, a policy $\pi \in \Pi$ interacting with POMDP $M$ induces a distribution $P_M^\pi$ over observation sequences $(o_1, \ldots, o_T) \in O^T$ and conditional suffix laws $P_M^\pi(O_{t+1:T}\mid h)$.

\section{Bounded Indistinguishability and the Wasserstein Pseudometric}
\label{sec:pseudometric}

Total variation is topologically brittle for model comparison \citep{ferns2004metrics}: it treats all non-identical observations as equally distant. When observations carry ordinal or spatial structure, TV forces overly fine partitions. The 1-Wasserstein distance $\cW_1$ exploits a ground metric $d_O$ on $O$ and degrades gracefully under perturbations. Let $d_O \colon O \times O \to \bR_{\geq 0}$ be a ground metric on observations, extended to sequences via $d_{O^T}((o_1,\ldots,o_T),(o_1',\ldots,o_T')) = \sum_{t=1}^T d_O(o_t, o_t')$.

\begin{definition}[Closed-loop Wasserstein pseudometric]
\label{def:pseudometric}
For two POMDPs $M$ and $N$ sharing $(A, O)$ and a fixed probe family $\Pi$:
\begin{equation}
\label{eq:pseudometric}
D_T^{\Pi}(M,N) := \sup_{\pi \in \Pi} \cW_1\bigl(P_M^\pi, P_N^\pi\bigr).
\end{equation}
When $\Pi = \Piclk$ or $\Pi = \Piop$, we write $D_{m,T}^{\mathrm{clk}}$ or $D_{m,T}^{\mathrm{op}}$ respectively.
\end{definition}

\begin{proposition}
\label{prop:pseudometric}
$D_T^\Pi$ is a pseudometric on POMDPs sharing $(A,O)$, with $D_T^\Pi(M,N) = 0$ iff $P_M^\pi = P_N^\pi$ for all $\pi \in \Pi$.
\end{proposition}
\begin{proof}
For each fixed $\pi \in \Pi$, the quantity $\cW_1(P_M^\pi, P_N^\pi)$ is a metric on laws over $O^T$. Non-negativity and symmetry therefore hold pointwise. For the triangle inequality, for any third model $K$ and every $\pi \in \Pi$,
\[
\cW_1(P_M^\pi, P_K^\pi)
\leq
\cW_1(P_M^\pi, P_N^\pi) + \cW_1(P_N^\pi, P_K^\pi),
\]
and taking the supremum over $\pi$ gives $D_T^\Pi(M,K) \leq D_T^\Pi(M,N) + D_T^\Pi(N,K)$. Finally, $D_T^\Pi(M,N)=0$ iff every term in the supremum is zero, equivalently $P_M^\pi=P_N^\pi$ for all $\pi \in \Pi$.
\end{proof}

\begin{remark}[Relation to bisimulation metrics]
\label{rem:bisim-comparison}
The conceptual lineage to Wasserstein bisimulation metrics is intentional. Both constructions use optimal-transport distances to quantify behavioral distinguishability. The difference is where the supremum lives and what object is being compared. State-based bisimulation metrics compare one-step transition kernels and reward terms through a fixed-point recursion over all actions. Our $D_T^\Pi$ instead compares finite-horizon observation-sequence laws after concrete histories and takes the supremum only over a prescribed bounded controller family. In that sense $D_T^\Pi$ is best viewed as a history-space, bounded-agent analogue of bisimulation metrics; the new contribution is that this bounded-family pseudometric induces canonical quotients and exact sufficiency statements for the chosen observer class.
\end{remark}

\begin{definition}[$L_R$-observation-Lipschitz reward]
\label{def:obs-lipschitz}
A reward function $R$ is \emph{$L_R$-observation-Lipschitz} if, for all equal-length histories $h, h'$ and all $\pi$ in the relevant fixed probe family,
$|\bar{R}(h,\pi) - \bar{R}(h',\pi)| \leq L_R \cdot \cW_1\bigl(P_M^\pi(O_{|h|+1:T} \mid h),\, P_M^\pi(O_{|h|+1:T} \mid h')\bigr)$,
where $\bar{R}(h,\pi) \coloneqq \sum_s b_h(s) \sum_a \pi(a \mid h)\, R(s,a)$.
Constant rewards satisfy $L_R = 0$; the standard Tiger reward (listen $= {-}1$, correct open $= {+}10$, incorrect open $= {-}100$) satisfies $L_R \leq 110$ (the reward range $R_{\max} - R_{\min}$); the synthetic 1-Lipschitz observation score used in the value-bound experiments (\S\ref{sec:experiments}) has $L_R = 1$.
\end{definition}

\begin{remark}[When is $L_R$ small?]
\label{rem:lipschitz-small}
Observation-based rewards $R(s,a) = r(o_s)$ for some function $r \colon O \to \bR$ satisfy $L_R \leq \mathrm{Lip}(r)$, the Lipschitz constant of $r$ under $d_O$.
When $r$ is bounded and $d_O$ is normalized to $[0,1]$, this gives $L_R \leq R_{\max} - R_{\min}$---the reward range.
More generally, any reward depending on the state only through the observation posterior has bounded $L_R$; rewards that depend on fine-grained latent state distinctions invisible to observations can have arbitrarily large $L_R$.
The value bound $L_R T \varepsilon$ is thus most informative for observation-aligned reward structures---precisely the setting where bounded agents are effective.
Note that $L_R$ depends on the choice of probe family through the conditional observation laws $P_M^\pi(O_{t+1:T} \mid h)$; a richer probe family may expose reward differences that a poorer one cannot, changing the effective Lipschitz constant.
In general, the reward range $R_{\max} - R_{\min}$ is a valid (but often vacuous) upper bound on $L_R$ under the discrete observation metric $d_O(o,o') = \mathbf{1}_{o \neq o'}$, since $\cW_1$ under the summed discrete metric upper-bounds total variation.
\end{remark}

\begin{theorem}[Value-function error bound]
\label{thm:value-bound}
Let $\widetilde{M}$ be the quotient POMDP (Definition~\ref{def:quotient}) constructed from a partition of histories into classes with pairwise probe-distance at most $\varepsilon$ under a fixed probe family $\Pi$, with beliefs aggregated within each class.
If $R$ is $L_R$-observation-Lipschitz, then for any $\pi \in \Pi$:
\begin{equation}
\label{eq:value-bound}
\bigl|V_M^\pi - V_{\widetilde{M}}^\pi\bigr| \leq L_R \cdot T \cdot \varepsilon.
\end{equation}
\end{theorem}
\begin{proof}
Fix $\pi \in \Pi$ and let $H_t$ denote the random history up to time $t$. For each stage define
\[
\Delta_t(H_t)
:=
\cW_1\!\bigl(P_M^\pi(O_{t+1:T}\mid H_t),\, P_{\widetilde{M}}^\pi(O_{t+1:T}\mid H_t)\bigr).
\]
The observation-Lipschitz assumption gives an almost-sure stagewise bound
\[
\bigl|\bar{R}_M(H_t,\pi)-\bar{R}_{\widetilde{M}}(H_t,\pi)\bigr|
\leq L_R\,\Delta_t(H_t).
\]
By hypothesis, the quotient $\widetilde{M}$ partitions histories into classes of pairwise probe-distance at most~$\varepsilon$.
Hence the quotient's conditional law $P_{\widetilde{M}}^\pi(O_{t+1:T}\mid H_t)$ is a convex combination (via the belief-aggregation weights) of conditional laws $\{P_M^\pi(O_{t+1:T}\mid h') : h' \in [H_t]\}$, each of which satisfies $\cW_1(P_M^\pi(\cdot\mid H_t), P_M^\pi(\cdot\mid h')) \leq \varepsilon$.
Because $\cW_1$ is convex (as a supremum of linear functionals), the mixture satisfies
\[
\Delta_t(H_t)\le \varepsilon
\qquad\text{a.s. for each }t.
\]
Taking expectations and summing the $T$ stagewise differences yields
\[
\bigl|V_M^\pi - V_{\widetilde{M}}^\pi\bigr|
\leq
L_R \sum_{t=1}^T \bE[\Delta_t(H_t)]
\leq
L_R\, T\, \varepsilon,
\]
which is exactly \eqref{eq:value-bound}.
\end{proof}

The bound is most informative when rewards align with the observation structure. For Tiger's standard latent-state reward ($L_R{=}110$), the ${\pm}100$ penalty depends on which door hides the tiger---a latent distinction invisible to observation sequences---making the bound explicitly vacuous. For observation-aligned rewards ($L_R \leq R_{\max} - R_{\min}$), the bound is tight, and the exact quotient preserves value perfectly (Theorem~\ref{thm:exact-sufficiency}).

\begin{corollary}[Sim-to-real regret]
\label{cor:sim-to-real}
Under the same hypotheses as Theorem~\ref{thm:value-bound}, let $\pi^*_{\widetilde{M}} \in \argmax_{\pi \in \Pi} V_{\widetilde{M}}^\pi$.
Then the regret of the quotient-optimal policy on the original model satisfies
\[
V_M^{\pi^*_M} - V_M^{\pi^*_{\widetilde{M}}} \leq 2\, L_R \cdot T \cdot \varepsilon.
\]
\end{corollary}
\begin{proof}
By the triangle inequality:
$V_M^{\pi^*_M} - V_M^{\pi^*_{\widetilde{M}}}
\leq (V_M^{\pi^*_M} - V_{\widetilde{M}}^{\pi^*_M}) + (V_{\widetilde{M}}^{\pi^*_{\widetilde{M}}} - V_M^{\pi^*_{\widetilde{M}}})
\leq 2\,L_R T \varepsilon$.
\end{proof}

\section{Equivalence on Histories and the Quotient POMDP}
\label{sec:quotient}

\subsection{Quotient Construction}

\begin{definition}[Bounded indistinguishability]
\label{def:equiv}
Fix a probe family $\Pi$. For equal-length histories $h, h' \in O^t$, define the \emph{history-level probe distance}
\[
d^\Pi(h,h') := \sup_{\pi \in \Pi} \cW_1(P_M^\pi(O_{t+1:T} \mid h), P_M^\pi(O_{t+1:T} \mid h')).
\]
Two histories are \emph{$\Pi$-equivalent}, written $h \equiv^{\Pi} h'$, if
\[
d^\Pi(h,h') = 0.
\]
For the two main families we write $\equiv^{\mathrm{clk}}_{m,T}$ and $\equiv^{\mathrm{op}}_{m,T}$, and $d^{\mathrm{clk}}_{m,T}(h,h')$, $d^{\mathrm{op}}_{m,T}(h,h')$ for the corresponding history-level distances.
\end{definition}

\begin{proposition}[Properties of $\equiv_{m,T}$]
\label{prop:equiv-props}
For every fixed probe family $\Pi$, the relation $\equiv^\Pi$ is (a) an equivalence relation, (b) right-invariant: $h \equiv^\Pi h'$ implies $h \cdot z \equiv^\Pi h' \cdot z$, and (c) of finite index.
\end{proposition}
\begin{proof}
Reflexivity and symmetry are immediate from the definition. For transitivity, if $h \equiv^\Pi h'$ and $h' \equiv^\Pi h''$, then for every $\pi \in \Pi$ the triangle inequality for $\cW_1$ gives zero distance between the conditional suffix laws from $h$ and $h''$. For right-invariance, fix $z \in O$. Zero Wasserstein distance between the suffix laws from $h$ and $h'$ implies equality of those conditional laws on every cylinder event, hence the same probability of seeing $z$ next and the same law of the remaining suffix after conditioning on that common event. Therefore $h\!\cdot\! z$ and $h'\!\cdot\! z$ induce identical continuation laws for every controller in $\Pi$, so $h \cdot z \equiv^\Pi h' \cdot z$. Finite index holds because at each depth $t$ there are only finitely many histories in $O^t$.
\end{proof}

\begin{definition}[Quotient POMDP]
\label{def:quotient}
For a fixed probe family $\Pi$, the quotient $Q^\Pi(M)$ has state space $\{[h]_\Pi : h \in O^t\}$ at time $t$, initial state $[\epsilon]_\Pi$, and transition
$\bar{P}([h], a, [h \cdot z]) = \sum_{s'} Z(s',a,z) \sum_{s} P(s,a,s')\, \bar{b}_{[h]}(s)$,
where $\bar{b}_{[h]}$ is the aggregated belief state of the class (any convex combination of member beliefs; by Theorem~\ref{thm:myhill-nerode}(i), the quotient transition is independent of this choice in the exact case).
The quotient is a finite-horizon controlled process with time-varying state space; equivalently, one can augment the state with the time index $t$ to obtain a stationary representation.
\end{definition}

\begin{theorem}[Bounded-Interaction Myhill--Nerode]
\label{thm:myhill-nerode}
Let $Q^\Pi(M)$ be the quotient under probe-exact equivalence for a fixed probe family $\Pi$.
Then:
\begin{enumerate}[label=(\roman*),leftmargin=2em,itemsep=0pt]
\item \textbf{Well-definedness.} The quotient transition kernel is independent of the choice of representative and aggregation weights.
\item \textbf{Soundness.} $P_M^\pi(O^T) = P_{Q^\Pi(M)}^\pi(O^T)$ for every $\pi \in \Pi$.
\item \textbf{Universality.} Any POMDP $N$ with $P_N^\pi = P_M^\pi$ for all $\pi \in \Pi$ admits a POMDP morphism $\phi \colon N \twoheadrightarrow Q^\Pi(M)$.
\item \textbf{Minimality.} $Q^\Pi(M)$ has the fewest history classes (equivalence classes of the observation-history tree) among all history-based POMDPs satisfying~(ii).
\item \textbf{Uniqueness.} $Q^\Pi(M)$ is unique up to isomorphism.
\end{enumerate}
\end{theorem}

\begin{proofsketch}
The invariant is: two representatives of the same class induce the same family of probe suffix laws after every continuation. Well-definedness follows because probe-exact equivalence forces identical one-step observation laws for all histories in a class; in particular, constant-action FSCs already imply equality of the one-step kernels needed to define $\bar{P}([h],a,\cdot)$, so the quotient transition is representative-independent. Soundness is then an induction on the realized history length: if the current quotient class matches the original history's probe law at time $t$, the common one-step kernel preserves that match after observing $z$ and moving to $h\!\cdot\! z$. For universality, any POMDP $N$ reproducing the same probe laws determines a canonical map sending each history of $N$ to the unique quotient class with the same family of suffix laws; the same representative-independence argument verifies the morphism identities. Minimality and uniqueness are the usual universal-object consequence: every competing exact model surjects onto $Q^\Pi(M)$, and two minimal such objects therefore surject onto one another and are isomorphic.
\end{proofsketch}

\subsection{Structural Properties}

\begin{definition}[Agent-accessible objectives]
\label{def:gobs}
Let $\Gobs$ be the class of bounded measurable functionals
\[
G \colon (O \times A)^T \to \bR
\]
defined on the joint observation-action trajectory.
\end{definition}
This is the natural class for bounded agents whose costs depend on what they sense and do: sensorimotor penalties, action budgets, tracking objectives defined on filtered observations, or other control costs measurable from the observation--action trace itself. When the reward depends on latent state variables that are not measurable from that trace, exact preservation is no longer available and the paper deliberately falls back to the observation-Lipschitz approximation story.

\begin{theorem}[Exact sufficiency for agent-accessible objectives]
\label{thm:exact-sufficiency}
For every $\pi \in \Piclk$ and every $G \in \Gobs$,
\[
\bE_M^\pi[G(O_{1:T}, A_{1:T})]
=
\bE_{\Qclk(M)}^\pi[G(O_{1:T}, A_{1:T})].
\]
\end{theorem}
\begin{proofsketch}
The argument proceeds in three stages: observation-law preservation, action-law transfer, and functional integration.

\emph{Stage~1 (Observation law).}
By the soundness part of Theorem~\ref{thm:myhill-nerode}, $\Qclk(M)$ reproduces the full observation law $P_M^\pi(O_{1:T})$ for every $\pi \in \Piclk$.
In particular, for every history $h = (o_1,\ldots,o_t)$ and every $\pi$, the conditional future-observation law $P_M^\pi(O_{t+1:T} \mid h) = P_{\Qclk(M)}^\pi(O_{t+1:T} \mid h)$.

\emph{Stage~2 (Action law).}
Each clock-aware policy $\pi \in \Piclk$ selects its stage-$t$ action as a deterministic (or stochastic) function of the observation history $o_{1:t}$ and the current internal state, which itself evolves deterministically from $o_{1:t}$.
The action at stage~$t$ is therefore measurable with respect to the observation history up to time~$t$.
Since Stage~1 guarantees identical conditional observation laws at every stage under $\pi$, an induction on $t = 1,\ldots,T$ shows that the joint marginals $P_M^\pi(O_{1:t}, A_{1:t})$ and $P_{\Qclk(M)}^\pi(O_{1:t}, A_{1:t})$ agree at every stage: at the inductive step, the common observation law at stage $t{+}1$ and the common policy mapping together fix $A_{t+1}$.

\emph{Stage~3 (Functional integration).}
Because $G \in \Gobs$ is a bounded measurable functional on $(O \times A)^T$, the expectation $\bE^\pi[G(O_{1:T},A_{1:T})]$ depends only on the joint law established in Stage~2. Since that law is identical in $M$ and $\Qclk(M)$, exact preservation follows.
\end{proofsketch}

\begin{corollary}[Optimal value preservation on $\Gobs$]
\label{cor:exact-sufficiency}
For every $G \in \Gobs$,
\[
\sup_{\pi \in \Piclk}\bE_M^\pi[G]
=
\sup_{\pi \in \Piclk}\bE_{\Qclk(M)}^\pi[G].
\]
\end{corollary}

\begin{remark}[Refinement and bisimulation recovery]
\label{rem:refinement}
The quotient family still forms a refinement lattice indexed by probe capacity, horizon, and approximation level, and the classical bisimulation quotient is recovered in the unbounded limit. The key point is different: the theorem object is the clock-aware quotient $\Qclk$, while the deterministic-stationary family used operationally in the experiments is a deliberate coarsening rather than a theorem-certified substitute.
\end{remark}

\begin{theorem}[Witness for clock-aware bounded agents]
\label{thm:witness}
If two histories $h, h'$ are distinguishable by some stochastic clock-aware FSC $\pi \in \Pi_{m,T}^{\mathrm{clk}}$, then there exists a \emph{deterministic} clock-aware FSC $\pi^* \in \Pi_{m,T}^{\mathrm{clk}}$ with $P_M^{\pi^*}(\cdot \mid h) \neq P_M^{\pi^*}(\cdot \mid h')$.
\end{theorem}
\begin{proofsketch}
For the clock-aware class, the probability of any fixed future observation sequence is multilinear in the stage-indexed FSC parameters; stage indexing prevents any parameter from being reused at two different times.
By a vertex lemma on the corresponding product of simplices, non-vanishing at any interior point implies non-vanishing at a vertex, i.e.\ a deterministic clock-aware FSC.
Proposition~\ref{prop:stationary-counterexample} shows that this argument does not extend to stationary looping FSCs; the experiments use deterministic stationary FSCs as an operational probe family only.
\end{proofsketch}

\begin{proposition}[Deterministic stationary FSCs do not suffice in general]
\label{prop:stationary-counterexample}
There exists a finite POMDP and histories $h,h'$ such that some stochastic stationary $1$-node FSC distinguishes $h$ and $h'$, while every deterministic stationary $1$-node FSC induces the same future observation law from $h$ and $h'$.
\end{proposition}
\begin{proofsketch}
Appendix~\ref{app:proofs} gives an explicit $m{=}1$, $T{=}3$ construction with actions $\{A,B\}$ and observations $\{L,R,U,X,Y\}$.
After histories $L$ and $R$, the two deterministic stationary controllers (always-$A$ and always-$B$) both induce the same suffix law, but the stochastic controller with $\alpha(A)=\alpha(B)=1/2$ yields suffix laws
$\{UU:3/4, UX:1/4\}$ and $\{UU:3/4, UY:1/4\}$ respectively, whose $\cW_1$ distance under the discrete metric is $1/4$.
\end{proofsketch}

The construction uses $m{=}1$; analogous counterexamples exist for general $m \geq 1$ by embedding the same branching structure within a larger node space, since the additional nodes cannot compensate for the loss of stochasticity when the branching occurs at the single active node.

\begin{proposition}[Smaller probe classes induce coarser quotients]
\label{prop:probe-class-coarsening}
Let $\Pi_1 \subseteq \Pi_2$ be two probe families. Then for every pair of histories,
\[
d^{\Pi_1}(h,h') \le d^{\Pi_2}(h,h').
\]
Consequently,
\[
h \equiv^{\Pi_2} h' \implies h \equiv^{\Pi_1} h',
\]
so the quotient $Q^{\Pi_1}(M)$ is a coarsening of $Q^{\Pi_2}(M)$. In particular, because $\Piop \subset \Piclk$, the operational quotient $\Qop$ used in the experiments is a tractable coarsening of the theorem-level quotient $\Qclk$.
\end{proposition}

\begin{theorem}[Cross-family value transfer]
\label{thm:cross-family-transfer}
Let $\Pi_1 \subseteq \Pi_2$ be probe families, and define the \emph{cross-family gap}
\[
\delta_{1,2} := \sup_{h,h'} \bigl(d^{\Pi_2}(h,h') - d^{\Pi_1}(h,h')\bigr)
= \|d^{\Pi_2} - d^{\Pi_1}\|_\infty .
\]
Let $\widetilde{M}_1$ be the $\varepsilon$-quotient built under $\Pi_1$.
Then for every $L_R$-observation-Lipschitz reward and every $\pi \in \Pi_2$,
\[
\bigl|V_M^\pi - V_{\widetilde{M}_1}^\pi\bigr| \le L_R\, T\,(\varepsilon + \delta_{1,2}).
\]
\end{theorem}
\begin{proof}
If $d^{\Pi_1}(h,h') \le \varepsilon$, then by definition of $\delta_{1,2}$,
\[
d^{\Pi_2}(h,h') \le d^{\Pi_1}(h,h') + \delta_{1,2} \le \varepsilon + \delta_{1,2}.
\]
Hence the partition built under $\Pi_1$ has pairwise $\Pi_2$-diameter at most $\varepsilon + \delta_{1,2}$. Applying Theorem~\ref{thm:value-bound} with probe family $\Pi_2$ and merge threshold $\varepsilon + \delta_{1,2}$ gives the result.
\end{proof}

\begin{corollary}[Operational-to-clock-aware transfer]
\label{cor:op-to-clk-transfer}
Let
\[
\delta_{\mathrm{clk}} := \|d^{\mathrm{clk}}_{m,T} - d^{\mathrm{op}}_{m,T}\|_\infty .
\]
Then for every $L_R$-observation-Lipschitz reward and every $\pi \in \Piclk$,
\[
\bigl|V_M^\pi - V_{\Qop(M)}^\pi\bigr| \le L_R\, T\,(\varepsilon + \delta_{\mathrm{clk}})
\]
whenever $\Qop(M)$ is constructed as an $\varepsilon$-quotient under $\Piop$ and the gap $\delta_{\mathrm{clk}}$ is measured for that same bounded regime.
\end{corollary}

\begin{corollary}[Subset plus cross-family transfer]
\label{cor:subset-cross-family-transfer}
Let $d_S(h,h') := \sup_{\pi \in S} \cW_1(P_M^\pi(O_{t+1:T}\mid h), P_M^\pi(O_{t+1:T}\mid h'))$ be the subset probe envelope for $S \subseteq \Pi_1$, and let $\delta_S := \|d^{\Pi_1} - d_S\|_\infty$. Then for every $L_R$-observation-Lipschitz reward and every $\pi \in \Pi_2$,
\[
\bigl|V_M^\pi - V_{\widetilde{M}_S}^\pi\bigr| \le L_R\, T\,(\varepsilon + \delta_S + \delta_{1,2}).
\]
\end{corollary}
\begin{proof}
If $d_S(h,h') \le \varepsilon$, then $d^{\Pi_1}(h,h') \le \varepsilon + \delta_S$ and therefore
\[
d^{\Pi_2}(h,h') \le d^{\Pi_1}(h,h') + \delta_{1,2} \le \varepsilon + \delta_S + \delta_{1,2}.
\]
Applying Theorem~\ref{thm:value-bound} with probe family $\Pi_2$ and merge threshold $\varepsilon + \delta_S + \delta_{1,2}$ yields the bound.
\end{proof}

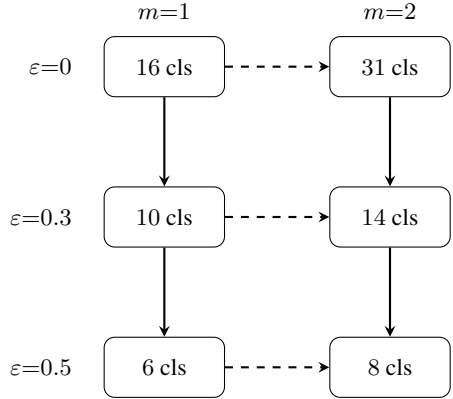
\begin{figure}[htbp]
\centering
\begin{tikzpicture}[
    node distance=2cm and 2.5cm,
    every node/.style={draw, rounded corners, minimum width=1.6cm, minimum height=0.8cm, align=center, font=\small},
    arrow/.style={->, >=stealth, thick}
]
\node (m1e5) at (0, 0) {$6$ cls};
\node (m2e5) at (3, 0) {$8$ cls};

\node (m1e3) at (0, 2) {$10$ cls};
\node (m2e3) at (3, 2) {$14$ cls};

\node (m1e0) at (0, 4) {$16$ cls};
\node (m2e0) at (3, 4) {$31$ cls};

\node[draw=none, minimum width=0pt, minimum height=0pt, left=0.3cm of m1e0, font=\footnotesize] {$\varepsilon{=}0$};
\node[draw=none, minimum width=0pt, minimum height=0pt, left=0.3cm of m1e3, font=\footnotesize] {$\varepsilon{=}0.3$};
\node[draw=none, minimum width=0pt, minimum height=0pt, left=0.3cm of m1e5, font=\footnotesize] {$\varepsilon{=}0.5$};

\node[draw=none, minimum width=0pt, minimum height=0pt, above=0.1cm of m1e0, font=\footnotesize] {$m{=}1$};
\node[draw=none, minimum width=0pt, minimum height=0pt, above=0.1cm of m2e0, font=\footnotesize] {$m{=}2$};

\draw[arrow] (m1e0) -- (m1e3);
\draw[arrow] (m1e3) -- (m1e5);
\draw[arrow] (m2e0) -- (m2e3);
\draw[arrow] (m2e3) -- (m2e5);

\draw[arrow, dashed] (m1e0) -- (m2e0);
\draw[arrow, dashed] (m1e3) -- (m2e3);
\draw[arrow, dashed] (m1e5) -- (m2e5);

\end{tikzpicture}
\caption{Refinement lattice for Tiger ($T{=}4$). Solid arrows: increasing $\varepsilon$ coarsens the quotient (Definition~\ref{def:eps-equiv}). Dashed arrows: increasing $m$ refines the quotient (Proposition~\ref{prop:probe-class-coarsening}---larger probe class $\Rightarrow$ finer partition). Data from operational capacity sweep.}
\label{fig:refinement-lattice}
\end{figure}

\section{\texorpdfstring{$\varepsilon$}{epsilon}-Quotients and Data-Processing Monotonicity}
\label{sec:wrappers}

Exact indistinguishability is often overly restrictive; we relax the framework to $\varepsilon$-approximate quotients.
We write $D_{m,T}^{\cW} := D_T^{\Pi}$ when the probe family $\Pi$ and the ground metric are clear from context, to emphasize the Wasserstein dependence.

\begin{definition}[$\varepsilon$-equivalence]
\label{def:eps-equiv}
$M \sim_{m,T}^\varepsilon N$ if $D_{m,T}^{\cW}(M,N) \leq \varepsilon$.
An \emph{$\varepsilon$-quotient} of $M$ is a reduced POMDP $\widetilde{M}$ with $D_{m,T}^{\cW}(M,\widetilde{M}) \leq \varepsilon$.
\end{definition}

In the approximate setting, the canonical object is the partition of histories. Materializing a quotient POMDP additionally requires a belief-aggregation rule inside each class; Proposition~\ref{prop:canonical-quotient} fixes the uniform canonical choice used throughout the experiments and bounds.

\begin{definition}[Wrapper]
\label{def:wrapper}
A wrapper $W$ maps a POMDP $M$ to $W(M)$ via an action remapping $g \colon A' \to \Delta(A)$ and an observation channel $C \colon O \to \Delta(O')$ with Lipschitz constant $L_C := \max_{o \neq o'} \cW_1(C(o), C(o');\, d_{O'}) / d_O(o, o')$.
\end{definition}

\begin{theorem}[Data-Processing Monotonicity]
\label{thm:data-processing}
Let $M, N$ be POMDPs sharing $(A,O)$, $W$ a wrapper with Lipschitz constant $L_C$, and assume the probe family $\Pi$ is closed under pullback through $W$ (i.e., for every $\pi' \in \Pi$ on the wrapped observation--action space, the induced controller $\pi$ on the original space belongs to $\Pi$; this holds for any family containing all $m$-bounded stochastic FSCs, in particular for $\Piclk$).
Then:
\begin{equation}
D_{m,T}^{\cW}\bigl(W(M), W(N)\bigr) \leq L_C \cdot D_{m,T}^{\cW}(M, N).
\end{equation}
In particular, equivalent models remain equivalent after wrapping.
\end{theorem}
\begin{proofsketch}
The wrapper post-processes observation sequences through $C^{\otimes T}$.
By Kantorovich--Rubinstein duality and the Lipschitz property, $\cW_1$ contracts by factor $L_C$.
Taking the same supremum over the bounded controller family on both sides preserves that contraction, so equivalence is stable under wrapping.
\end{proofsketch}

\begin{definition}[$\delta_O$-coarsened agent class]
\label{def:delta-coarsening}
Given a ground metric $d_O$ on observations and resolution $\delta_O \geq 0$,
let $O_\delta \subseteq O$ be a minimal $\delta_O$-covering:
every $o \in O$ satisfies $\min_{o' \in O_\delta} d_O(o, o') \leq \delta_O$.
The \emph{$\delta_O$-coarsened agent class} $\Pi_{m,T,\delta}$ consists of
FSCs whose node-transition function factors through the quantization map
$q_\delta \colon O \to O_\delta$ defined by $q_\delta(o) = \argmin_{o' \in O_\delta} d_O(o, o')$.
\end{definition}

\begin{proposition}[Observation-resolution bounds]
\label{prop:obs-resolution}
Let $Q_{m,T,\delta}$ denote the $\varepsilon$-quotient under $\Pi_{m,T,\delta}$.
\begin{enumerate}[label=(\alph*),itemsep=1pt]
\item \emph{Distance bound:}
$D_{m,T,\delta}^{\cW}(M,\widetilde{M}) \leq D_{m,T}^{\cW}(M,\widetilde{M}) + T \cdot \delta_O$.
\item \emph{Partition size:}
$Q_{m,T,\delta}$ has at most $\sum_{t=0}^T |O_\delta|^t$ equivalence classes.
\item \emph{Value error:}
$|V_M^\pi - V_{Q_{m,T,\delta}}^\pi| \leq L_R \cdot T \cdot (\varepsilon + T\delta_O)$
for all $\pi \in \Pi_{m,T,\delta}$.
\end{enumerate}
\end{proposition}
\begin{proofsketch}
Part~(a): replacing each observation by its $\delta_O$-nearest representative shifts
the per-step $\cW_1$ by at most $\delta_O$; summing over $T$ steps gives the additive
$T\delta_O$ term.
Part~(b): the coarsened history space has $|O_\delta|^t$ histories at depth~$t$.
Part~(c): the partition under $\Pi_{m,T,\delta}$ has $\Pi_{m,T,\delta}$-diameter at most $\varepsilon$; the same per-step $\delta_O$ shift from part~(a) applies at the history level to conditional suffix distances, so the partition has $\Pi_{m,T}$-diameter at most $\varepsilon + T\delta_O$. Theorem~\ref{thm:value-bound} then applies with merge threshold $\varepsilon + T\delta_O$.
\end{proofsketch}

\begin{theorem}[Compositional horizon-scaling bound]
\label{thm:layered-composition}
Let $T$ be the end-to-end horizon.
Let reference models $(M_i)_{i=1}^{L+1}$, approximate models $(\widetilde{M}_i)_{i=1}^{L+1}$, and wrappers
$W_i$ with Lipschitz constants $L_i$ satisfy
\begin{equation}
\Gamma_1 := D_{m,T}^{\cW}(M_1,\widetilde{M}_1),\qquad
M_{i+1}=W_i(M_i),\qquad
D_{m,T}^{\cW}\!\bigl(W_i(\widetilde{M}_i),\widetilde{M}_{i+1}\bigr)\le \varepsilon_i
\quad (i=1,\ldots,L).
\end{equation}
Define $\Gamma_i := D_{m,T}^{\cW}(M_i,\widetilde{M}_i)$.
Then the cumulative distortion obeys
\begin{equation}
\label{eq:layered-recursion}
\Gamma_{i+1}\le L_i\Gamma_i+\varepsilon_i,\qquad i=1,\ldots,L,
\end{equation}
and hence the final-layer distortion satisfies
\begin{equation}
\label{eq:layered-composition}
\Gamma_{L+1}
\le
\left(\prod_{j=1}^{L} L_j\right)\Gamma_1
+
\sum_{i=1}^L \left(\prod_{j=i+1}^{L} L_j\right)\varepsilon_i.
\end{equation}
In the common zero-initialization case $\widetilde{M}_1=M_1$, one has $\Gamma_1=0$ and only the weighted sum of the per-layer residuals remains.
\end{theorem}
\begin{proofsketch}
The first term comes from propagating the previous layer's discrepancy through the wrapper, and the second from the fresh approximation introduced at the new layer.
Apply Theorem~\ref{thm:data-processing} and the triangle inequality to obtain~\eqref{eq:layered-recursion}, then unroll the recursion by induction.
\end{proofsketch}

\begin{corollary}[Layered value bound]
\label{cor:layered-value}
If additionally $\widetilde{M}_{L+1}$ is a partition-based quotient of $M_{L+1}$ with class diameter at most $\Gamma_{L+1}$ (i.e., the partition is chosen so that its pairwise probe-distance does not exceed the propagated bound from Theorem~\ref{thm:layered-composition}), and the reward on the final observation space is $L_R$-observation-Lipschitz,
then for any $\pi\in\Pi_{m,T}$:
\begin{equation}
\bigl|V_{M_{L+1}}^\pi - V_{\widetilde{M}_{L+1}}^\pi\bigr|
\le
L_R \cdot T \cdot
\left[
\left(\prod_{j=1}^{L} L_j\right)\Gamma_1
+
\sum_{i=1}^L \left(\prod_{j=i+1}^{L} L_j\right)\varepsilon_i.
\right]
\end{equation}
\end{corollary}
\begin{proofsketch}
Theorem~\ref{thm:layered-composition} bounds $\Gamma_{L+1}$; Theorem~\ref{thm:value-bound} applied with merge threshold $\Gamma_{L+1}$ gives the result.
\end{proofsketch}

\begin{proposition}[Layered construction complexity]
\label{prop:layered-complexity}
For fixed $m$ and uniform segment horizon $\tau$ with $T=L\tau$, monolithic computation has complexity
\begin{equation}
\label{eq:mono-complexity}
O\!\left(\frac{|O|^{T+1}-1}{|O|-1}\cdot |A|^m m^{m|O|}\cdot T\lvert S\rvert^2|O|\right),
\end{equation}
while layered construction costs
\begin{equation}
\label{eq:layered-complexity}
O\!\left(
L\cdot\frac{|O|^{\tau+1}-1}{|O|-1}\cdot |A|^m m^{m|O|}\cdot \tau\lvert S\rvert^2|O|
\right).
\end{equation}
Hence layering replaces the exponential history term in $T$ by a linear factor in $L$ times an exponential in $\tau$.
\end{proposition}

In the Options framework \citep{sutton1999between}, wrappers correspond to hierarchy layers: macro-actions are action remappings, abstract observations are channels.
Theorem~\ref{thm:layered-composition} gives the corresponding multi-layer error accumulation rule and Proposition~\ref{prop:layered-complexity} gives the associated horizon-scaling tradeoff.

\section{Operational Approximation Machinery}
\label{sec:spectral-approx}

This section is operational rather than theorem-level. Greedy subset selection gives a practical way to compress the deterministic-stationary probe class, but by itself it only optimizes a coverage objective. The role of $\delta_S$ is to turn a chosen subset into an \emph{a posteriori} operational certificate: once computed from the full deterministic-stationary FSC tensor on operational exact-for-family benchmarks, it measures how much probe-envelope information the subset misses and therefore whether the value guarantee survives the compression.

\begin{algorithm}[htbp]
\caption{Approximate Partition Refinement for $\varepsilon$-Quotient}
\label{alg:approx-partition}
\small
\begin{algorithmic}[1]
\Require POMDP $M$, bounds $m, T$, tolerance $\varepsilon > 0$
\Ensure $\varepsilon$-quotient $\widetilde{Q}$
\State Enumerate deterministic FSCs $\Pi_m^{\mathrm{det}}$
\State Initialize partition $\mathcal{P} \gets \{O^t : 0 \le t \le T\}$ \Comment{group by length}
\Repeat
  \For{each block $B \in \mathcal{P}$}
    \State Split $B$ by $\max_{\pi} \cW_1\bigl(P_M^\pi(\cdot \mid h),\, P_M^\pi(\cdot \mid h')\bigr) \le \varepsilon$
  \EndFor
\Until{$\mathcal{P}$ is stable}
\State \Return quotient POMDP with canonical beliefs $\bar{b}_{[h]}$
\end{algorithmic}
\end{algorithm}

Worst-case complexity is $O\!\left(\frac{|O|^{T+1}-1}{|O|-1} \cdot |A|^m \cdot m^{m|O|} \cdot T \lvert S\rvert^2 |O|\right)$---polynomial in $|S|$ for fixed $m, T$.
This probe-exact procedure has cost exponential in the horizon parameter $T$, because it expands a depth-$T$ history tree and evaluates depth-$T$ bounded-controller behavior explicitly; whether the problem itself requires exponential time remains open (Appendix~\ref{app:hardness}).

\paragraph{Structural vs.\ operational role.}
The classical Myhill--Nerode theorem for DFAs yields an $O(n \log n)$ minimization algorithm \citep{hopcroft1979introduction}; our bounded-interaction analogue does not currently come with a comparable polynomial-time algorithm for the probe-exact quotient, and the explicit procedure above remains exponential in the horizon parameter. The paper's main scientific claim is therefore structural: it identifies the canonical bounded-observer target and proves its uniqueness, minimality, and value-preservation properties. The operational machinery is the companion computational layer. It explains how tractable coarsenings such as $\Qop$ can be studied in practice through subset certificates, sampling, layering, and observation coarsening, but those operational studies only support theorem-level statements when an explicit subset or cross-family certificate is reported for the same bounded regime. Otherwise they should be read as empirical case studies of a tractable surrogate rather than as additional theorem validation, and the uncertified medium-scale tables should not be read as part of the paper's core contribution claims.
Appendix~\ref{app:structural-subset} records a stricter structural sufficient condition via anchor families.

\paragraph{Practical tractability.}
The four complexity factors---histories ($|O|^T$), FSCs ($|A|^m m^{m|O|}$), belief propagation ($|S|^2$), and observation-alphabet size ($|O|$)---are addressed by complementary techniques.
(1)~\emph{The history bottleneck is mitigated by layered horizon decomposition} (Theorem~\ref{thm:layered-composition}, Proposition~\ref{prop:layered-complexity}): replacing one $T$-step computation by $L$ short segments of length $\tau$ changes the dominant history factor from $|O|^T$ to $L|O|^\tau$.
(2)~\emph{The FSC bottleneck is mitigated by greedy controller-subset selection} (Algorithm~\ref{alg:greedy-fsc}): on the operational exact-for-family Tiger and GridWorld benchmarks, small subsets recover the full probe envelope after explicit \emph{a posteriori} checking ($\delta_S{=}0$; Table~\ref{tab:partition-agreement}), while the same heuristic also underlies the appendix-only archival stress tracks for larger-$m$ regimes (Appendix~\ref{sec:long-horizon-scaling}).
(3)~\emph{Sampling-based $\cW_1$ estimation} decouples the per-history cost from $|S|$: each trajectory simulation costs $O(T)$ per step after an $O(|S|)$ belief sample, enabling operational case studies up to $|S|{=}100$ in $0.30$\,s (Table~\ref{tab:medium-scale}).
(4)~\emph{The observation-alphabet bottleneck is reduced by $\delta_O$-coarsening} (Proposition~\ref{prop:obs-resolution}): replacing $|O|$ by $|O_\delta| \leq |O|$ reduces the history factor from $|O|^T$ to $|O_\delta|^T$ at the cost of an additive $L_R T^2 \delta_O$ term in the value bound.
Together, these techniques make the operational core experiments practical on a single CPU core, while the heavier long-horizon and larger-$m$ stress tracks are reported separately in Appendix~\ref{sec:long-horizon-scaling}.

\paragraph{Greedy controller-subset selection.}
The FSC distance tensor $D_{(i,j),p} = \cW_1(P_M^{\pi_p}(\cdot \mid h_i), P_M^{\pi_p}(\cdot \mid h_j))$ is operationally compressible: distinguishing power often concentrates on a small number of high-coverage controllers. We therefore select FSCs by maximizing the monotone submodular coverage objective $f(S) = \sum_{(i,j)} \max_{p \in S} D_{(i,j),p}$ via greedy selection.

\begin{algorithm}[htbp]
\caption{Greedy Controller-Subset Selection}
\label{alg:greedy-fsc}
\small
\begin{algorithmic}[1]
\Require Distinguishing matrix $\mathbf{D} \in \bR^{\binom{n}{2} \times P}$, budget $k$
\Ensure Index set $S \subseteq \{1, \ldots, P\}$ with $|S| = k$
\State $S \gets \emptyset$;\; $d^*_{(i,j)} \gets 0$ for all $(i,j)$
\For{$t = 1, \ldots, k$}
  \State $p^* \gets \argmax_{p \notin S} \sum_{(i,j)} \max(0,\, D_{(i,j),p} - d^*_{(i,j)})$
  \State $S \gets S \cup \{p^*\}$
  \State $d^*_{(i,j)} \gets \max(d^*_{(i,j)},\, D_{(i,j),p^*})$ for all $(i,j)$
\EndFor
\State \Return $S$
\end{algorithmic}
\end{algorithm}

\begin{proposition}[Submodular guarantee]
\label{prop:submodular}
The greedy selection achieves $f(S_k) \geq (1 - 1/e) \cdot f(S_k^*)$ \citep{nemhauser1978analysis}.
\end{proposition}

This $(1-1/e)$ guarantee is a coverage guarantee only: by itself it does not imply operational exact-partition recovery or value preservation.
Those stronger claims require measured small probe-envelope gap $\delta_S$, and Table~\ref{tab:partition-agreement} is empirical-plus-certified only because $\delta_S{=}0$ is explicitly checked from the full deterministic-stationary FSC tensor.

\paragraph{Operational certificate for controller subsets.}
Let $U := \{(i,j) : \text{history pairs}\}$ index the rows of the distance tensor, and write
$d(u) := \max_p D_{u,p}$ for the \emph{full probe envelope} and
$d_S(u) := \max_{p \in S} D_{u,p}$ for the \emph{subset probe envelope} under a selected controller subset~$S$.
Define the \emph{probe-envelope gap} of~$S$ by
\[
  \delta_S := \|d - d_S\|_\infty .
\]

\begin{theorem}[Uniform probe approximation]
\label{thm:probe-approx}
The $\varepsilon$-quotient built from~$S$ satisfies
\[
  |V_M^\pi - V_{\widetilde{M}_S}^\pi| \leq L_R\, T\,(\varepsilon + \delta_S)
  \qquad \forall\, \pi \in \Pi_{m,T}.
\]
\end{theorem}
\begin{proof}
Since $d_S(u) \leq d(u)$ for every~$u$, any pair merged by the full-family $\varepsilon$-quotient is also merged by the subset-based $\varepsilon$-quotient.
Conversely, if $d_S(u) \leq \varepsilon$, then $d(u) \leq d_S(u) + \delta_S \leq \varepsilon + \delta_S$,
so the subset-built partition has full-family diameter at most $\varepsilon + \delta_S$.
Applying Theorem~\ref{thm:value-bound} with merge threshold $\varepsilon + \delta_S$ gives the result.
\end{proof}

Appendix~\ref{app:structural-subset} gives a stricter structural sufficient condition for why a small controller subset may achieve small $\delta_S$: a $\delta$-anchor controller family guarantees small probe-envelope error.
We keep that result separate from the main flow because the experiments certify the operational quantity $\delta_S$ directly, not anchor-family membership.

\section{Experiments}
\label{sec:experiments}

This section separates theorem-aligned clock-aware exact experiments from operational results and then marks the heaviest scale-up rows as archival stress tests. The clock-aware exact experiments use deterministic clock-aware open-loop probes for tractable $m{=}1$ cases, which coincide with the theorem object for that bounded regime. The larger-scale deterministic-stationary experiments remain operational results for $\Qop$ and are included to characterise the behaviour of that tractable surrogate, not to extend the theorem's claim surface.
Only the clock-aware exact tables and the deterministic-stationary tables with an explicit $\delta_S{=}0$ certificate should be read as exact-for-family evidence; the larger scaling tables are intentionally presented as empirical operational demonstrations. Unless stated otherwise, theorem-aligned tables are exhaustive over the stated bounded deterministic clock-aware probe family, while operational scaling tables use deterministic-stationary probes. The new point is quantitative but local: when the cross-family gap $\delta_{\mathrm{clk}}$ is measured on a tractable exact case, Theorem~\ref{thm:cross-family-transfer} upgrades \emph{that measured case} to an additive value-transfer certificate for the theorem-level family. Where no such $\delta_{\mathrm{clk}}$ measurement is available, the scaling tables remain operational evidence about $\Qop$, not direct empirical verification of $\Qclk$, not part of the evidence for contribution~(v), and not part of the paper's theorem-facing claim set. Sampling-based runs use $500$ trajectories per history-policy pair, report $1{,}000$-resample bootstrap confidence intervals for $\max_{h,h'}\cW_1$, and include dedicated convergence and seed-stability checks (five replications for convergence, ten seeds for stability); benchmark definitions, expanded tables, parameter grids, and the command-to-table mapping are collected in Appendix~\ref{app:experiments}. All runtimes are single-core wall-clock measurements on a MacBook Pro with an Apple M3 Pro chip (12 cores, 36\,GB RAM), macOS~26.3.1, and Python~3.11.14; no GPU was used.

Table~\ref{tab:evidence-tiers} is the evidence contract for the remainder of the section. It separates theorem-level validation, operational exact-for-family certification, and operational empirical stress tests. The paper's theorem claims are supported only by Tier~I and, where applicable, Tier~II rows; Tier~III rows are reported as scalability evidence about the tractable operational quotient and should not be read as theorem validation unless an explicit cross-family or subset certificate is reported for that benchmark.

\begin{table}[htbp]
\centering
\caption{Evidence tiers used in the experiments section. This is the claim-to-evidence contract for the empirical narrative.}
\label{tab:evidence-tiers}
\footnotesize
\resizebox{\linewidth}{!}{\begin{tabular}{@{}p{0.16\linewidth}p{0.24\linewidth}p{0.28\linewidth}p{0.24\linewidth}@{}}
\toprule
\textbf{Tier} & \textbf{Representative results} & \textbf{What the tier supports} & \textbf{Certificate type} \\
\midrule
Tier I: theorem-aligned exact & Tables~\ref{tab:probe-family-comparison}, \ref{tab:observation-planning}, \ref{tab:latent-planning-exact} & Theorem-level statements about $\Qclk$ on the stated bounded exact cases, including measured cross-family transfer on tractable exact instances & Exhaustive clock-aware enumeration on the stated finite benchmark and horizon \\
Tier II: operational exact-for-family & Table~\ref{tab:partition-agreement} & Exact statements about the operational family $\Qop$ and controller-subset certification within that family & Full deterministic-stationary FSC tensor together with reported $\delta_S$ certificate \\
Tier III: operational empirical scaling & Tables~\ref{tab:medium-scale}, \ref{tab:meaningful-scale}, \ref{tab:hierarchical-scaling}, \ref{tab:principal-horizon} & Compression, stability, and runtime evidence for the tractable operational quotient $\Qop$; not theorem validation for $\Qclk$ unless an explicit bridge is reported & Sampling confidence intervals, ARI checks, runtime reporting, and explicit disclosure when no theorem-level certificate is available \\
\bottomrule
\end{tabular}}
\end{table}

\subsection{Clock-Aware vs.\ Operational Probe Families}

Table~\ref{tab:probe-family-comparison} compares the theorem-level clock-aware quotient $\Qclk$ with the operational deterministic-stationary quotient $\Qop$ on tractable clock-aware exact cases. The key point is formal rather than numerical: Proposition~\ref{prop:probe-class-coarsening} predicts that $\Qop$ is a coarsening of $\Qclk$, and the small clock-aware exact benchmarks quantify that gap directly. In these tractable rows, the reported column $\max|d^{\mathrm{clk}}-d^{\mathrm{op}}|$ is exactly the measured cross-family gap $\delta_{\mathrm{clk}}$ entering Theorem~\ref{thm:cross-family-transfer}.

\begin{table}[htbp]
\centering
\caption{Exact comparison between the theorem-level clock-aware quotient and the operational deterministic-stationary quotient. The stationary witness row is the didactic separation guaranteed by Proposition~\ref{prop:stationary-counterexample}. The column $\max|d^{\mathrm{clk}}-d^{\mathrm{op}}|$ reports the largest absolute pairwise $\cW_1$ difference (unnormalized); values exceeding $1$ arise because the sequence-level metric sums contributions over $T$ time steps.}
\label{tab:probe-family-comparison}
\scriptsize
\resizebox{\linewidth}{!}{\begin{tabular}{llrrrrr}
\toprule
Benchmark & $T$ & $|Q^{\mathrm{op}}|$ & $|Q^{\mathrm{clk}}|$ & ARI & $\max_{h,h'} |d^{\mathrm{clk}}-d^{\mathrm{op}}|$ & Obs.\ value changed?\\
\midrule
Tiger & 2 & 4 & 4 & 1.000 & 0.490 & No\\
Tiger & 4 & 11 & 16 & 0.961 & 1.315 & No\\
Tiger & 6 & 22 & 64 & 0.953 & 2.077 & No\\
Tiger & 8 & 37 & 256 & 0.956 & 2.795 & No\\
Tiger & 10 & 56 & 1024 & 0.959 & 3.499 & No\\
GridWorld 3x3 & 2 & 6 & 6 & 1.000 & 0.347 & No\\
GridWorld 3x3 & 3 & 22 & 22 & 1.000 & 0.549 & No\\
GridWorld 5x5 & 2 & 6 & 6 & 1.000 & 0.355 & No\\
Stationary witness & 3 & 9 & 10 & 0.957 & 1.000 & No\\
\bottomrule
\end{tabular}}
\end{table}

Three points matter. First, the operational family is not uniformly too weak: on GridWorld $3{\times}3$ (including $T{=}3$) and GridWorld $5{\times}5$ ($|S|{=}25$, $T{=}2$), it agrees with the theorem object exactly (ARI${=}1.0$, identical class counts). Second, the gap is not merely philosophical: the stationary witness benchmark shows that $\Qop$ can merge histories that $\Qclk$ must keep separate, and the reported pseudometric gap makes that mismatch quantitative. Third, the measured gap $\delta_{\mathrm{clk}}$ grows with horizon on Tiger (from $0.49$ at $T{=}2$ to $3.50$ at $T{=}10$) yet never changes the observation-value decision, confirming that the operational coarsening is conservative in practice even when the pseudometric gap is large. By Theorem~\ref{thm:cross-family-transfer}, that same measured gap is a value-transfer certificate on these tractable exact cases: an operational $\varepsilon$-quotient is also a valid $(\varepsilon + \delta_{\mathrm{clk}})$-quotient for the clock-aware family.

\paragraph{Multi-node probes ($m{=}2$).}
Moving beyond open-loop ($m{=}1$) controllers, we enumerate all $1{,}296$ deterministic clock-aware FSCs with $m{=}2$ nodes for Tiger at $T{=}2$ (Definition~\ref{def:fsc}, Section~\ref{sec:prelim}), compared to $147$ deterministic stationary $m{\leq}2$ FSCs. Despite the ${\approx}9{\times}$ richer clock-aware family, both produce identical probe-exact partitions ($4$ classes, ARI${=}1.0$, $\max|d^{\mathrm{clk}}-d^{\mathrm{op}}|{=}0$). This indicates that for Tiger at short horizons, the stationary family already saturates the discriminative power of the full clock-aware family---consistent with the intuition that Tiger's symmetric belief dynamics limit the additional resolution that stage-dependent parameters can provide.

\subsection{Decision Sufficiency for Bounded Planning}
\label{sec:decision-sufficiency}

The central theorem-level claim is not exact preservation for arbitrary rewards, but exact preservation for objectives in $\Gobs$ measurable on the joint observation-action trajectory. Table~\ref{tab:observation-planning} verifies that claim on the clock-aware exact cases---Tiger (horizons up to $T{=}10$), GridWorld $3{\times}3$ ($T{\leq}3$), and GridWorld $5{\times}5$ ($T{=}2$, $|S|{=}25$)---using observation-only and action-observation objectives. In every row, the quotient-selected policy matches the original-model optimum in value, as Theorem~\ref{thm:exact-sufficiency} predicts.

\begin{table}[htbp]
\centering
\caption{Exact preservation of bounded observation-action objectives under the clock-aware quotient $\Qclk$.}
\label{tab:observation-planning}
\scriptsize
\resizebox{\linewidth}{!}{\begin{tabular}{llrrrrrp{3.2cm}rrr}
\toprule
Benchmark & Obj. & $T$ & Hist. & Cls. & $t_{\mathrm{orig}}$ & $t_{Q^{\mathrm{clk}}}$ & Policy & $V_{\mathrm{orig}}$ & $V_{Q^{\mathrm{clk}}}$ & Regret\\
\midrule
Tiger & Obs score & 2 & 7 & 4 & 0.000 & 0.000 & L L & 1.000 & 1.000 & 0.000\\
Tiger & Action+obs score & 2 & 7 & 4 & 0.000 & 0.000 & OL OL & 1.000 & 1.000 & 0.000\\
Tiger & Obs score & 4 & 31 & 16 & 0.001 & 0.003 & L L L L & 2.000 & 2.000 & 0.000\\
Tiger & Action+obs score & 4 & 31 & 16 & 0.001 & 0.003 & OL OL OL OL & 2.000 & 2.000 & 0.000\\
Tiger & Obs score & 6 & 127 & 64 & 0.008 & 0.132 & L L L L L L & 3.000 & 3.000 & 0.000\\
Tiger & Action+obs score & 6 & 127 & 64 & 0.009 & 0.139 & OL OL OL OL OL OL & 3.000 & 3.000 & 0.000\\
Tiger & Obs score & 8 & 511 & 256 & 0.104 & 9.993 & L L L L L OL L L & 4.000 & 4.000 & 0.000\\
Tiger & Action+obs score & 8 & 511 & 256 & 0.100 & 10.023 & OL OL OL OL OL OL OL OL & 4.000 & 4.000 & 0.000\\
Tiger & Obs score & 10 & 2047 & 1024 & 1.132 & 1200.886 & L L L L L L L L L L & 5.000 & 5.000 & 0.000\\
Tiger & Action+obs score & 10 & 2047 & 1024 & 1.099 & 1194.791 & OL OL OL OL OL OL OL OL OL OL & 5.000 & 5.000 & 0.000\\
GridWorld 3x3 & Obs score & 2 & 21 & 6 & 0.001 & 0.000 & D D & 1.120 & 1.120 & 0.000\\
GridWorld 3x3 & Action+obs score & 2 & 21 & 6 & 0.000 & 0.000 & U U & 0.968 & 0.968 & 0.000\\
GridWorld 3x3 & Obs score & 3 & 85 & 22 & 0.006 & 0.008 & D D R & 1.787 & 1.787 & 0.000\\
GridWorld 3x3 & Action+obs score & 3 & 85 & 22 & 0.006 & 0.008 & L U U & 1.562 & 1.562 & 0.000\\
GridWorld 5x5 & Obs score & 2 & 21 & 6 & 0.000 & 0.000 & D D & 1.072 & 1.072 & 0.000\\
GridWorld 5x5 & Action+obs score & 2 & 21 & 6 & 0.000 & 0.000 & U U & 0.805 & 0.805 & 0.000\\
\bottomrule
\end{tabular}}
\end{table}

\subsection{Latent-State Planning on Exact Clock-Aware Quotients}
\label{sec:latent-state-planning}

The approximate-value story remains relevant for latent-state rewards that are not measurable with respect to the observation-action trajectory. Table~\ref{tab:latent-planning-exact} therefore reports the clock-aware exact $m{=}1$ planning study for Tiger's standard reward and a goal-reward GridWorld. These rows use the theorem object $\Qclk$, but only the approximate latent-state narrative applies.

\begin{table}[htbp]
\centering
\caption{Latent-state planning impact on clock-aware exact cases. Values are reported on the original model; regret is relative to the original-model optimum.}
\label{tab:latent-planning-exact}
\scriptsize
\resizebox{\linewidth}{!}{\begin{tabular}{llrrrrrp{3.2cm}rrr}
\toprule
Benchmark & Reward & $T$ & Hist. & Cls. & $t_{\mathrm{orig}}$ & $t_{Q^{\mathrm{clk}}}$ & Policy & $V_{\mathrm{orig}}$ & $V_{Q^{\mathrm{clk}}\rightarrow M}$ & Regret\\
\midrule
Tiger & Tiger reward & 2 & 7 & 4 & 0.000 & 0.000 & L L & -2.000 & -2.000 & 0.000\\
Tiger & Tiger reward & 4 & 31 & 16 & 0.000 & 0.009 & L L L L & -4.000 & -4.000 & 0.000\\
Tiger & Tiger reward & 6 & 127 & 64 & 0.005 & 0.386 & L L L L L L & -6.000 & -6.000 & 0.000\\
Tiger & Tiger reward & 8 & 511 & 256 & 0.061 & 19.515 & L L L L L L L L & -8.000 & -8.000 & 0.000\\
Tiger & Tiger reward & 10 & 2047 & 1024 & 0.663 & 1549.370 & L L L L L L L L L L & -10.000 & -10.000 & 0.000\\
GridWorld 3x3 & Goal reward & 2 & 21 & 6 & 0.000 & 0.001 & D U & 0.311 & 0.311 & 0.000\\
GridWorld 3x3 & Goal reward & 3 & 85 & 22 & 0.000 & 0.023 & D R U & 0.671 & 0.671 & 0.000\\
\bottomrule
\end{tabular}}
\end{table}

\subsection{Operational Controller Subsets and A Posteriori Certification}

Table~\ref{tab:partition-agreement} reports partition agreement and the \emph{a posteriori} probe-envelope certificate $\delta_S = \|d-d_S\|_\infty$ for greedy controller-subset selection (Algorithm~\ref{alg:greedy-fsc}) on the operational deterministic-stationary $m{=}2$ benchmarks.
Because $\delta_S$ is computed from the full deterministic-stationary FSC tensor, this certification is available only on the operational exact-for-family benchmarks in this subsection.
On Tiger ($147$ FSCs), $k{=}5$ yields $\delta_S{=}0$ and operational exact-partition recovery for all tested $\varepsilon$.
On GridWorld $3{\times}3$ ($6{,}405$ FSCs), $k{=}5$ likewise yields $\delta_S{=}0$; in fact $k{=}3$ already closes the probe gap to numerical precision.
As a descriptive spectral summary, the GridWorld $99\%$ effective rank (number of singular values capturing $99\%$ of total variance) is $5$, so less than $0.1\%$ of FSCs capture $99\%$ of distinguishing variance.
Reporting $\delta_S$ sharpens the empirical story: for example, on GridWorld at $\varepsilon{=}0.3$, $k{=}1$ already gives a high ARI, but its nonzero probe gap explains why the subset is not yet operationally certified by Theorem~\ref{thm:probe-approx}.

\begin{table}[htbp]
\centering
\caption{Operational partition recovery and probe-envelope certification on the deterministic-stationary $m{=}2$ benchmarks. Here $\delta_S := \|d-d_S\|_\infty$ is an \emph{a posteriori} certificate computed from the full deterministic-stationary FSC tensor.}
\label{tab:partition-agreement}
\small
\begin{tabular}{@{}lcccccc@{}}
\toprule
\textbf{Benchmark} & $\varepsilon$ & $k$ & $\delta_S$ & Op.\ exact classes & Approx.\ classes & ARI \\
\midrule
Tiger ($T{=}4$) & 0.0 & 1 & 0.980 & 16 & 11 & 0.961 \\
Tiger ($T{=}4$) & 0.0 & 3 & 0.245 & 16 & 15 & 0.994 \\
Tiger ($T{=}4$) & 0.0 & 5 & 0.000 & 16 & 16 & 1.000 \\
\midrule
GridWorld $3{\times}3$ ($T{=}2$) & 0.3 & 1 & 0.233 & 6 & 4 & 0.981 \\
GridWorld $3{\times}3$ ($T{=}2$) & 0.3 & 3 & 0.000 & 6 & 6 & 1.000 \\
GridWorld $3{\times}3$ ($T{=}2$) & 0.3 & 5 & 0.000 & 6 & 6 & 1.000 \\
\bottomrule
\end{tabular}
\end{table}

\subsection{Larger-Scale and Sensitivity Results}
\label{sec:larger-scale-and-sensitivity-results}

\paragraph{Operational exact-for-family computation ($|S| \leq 36$).}
On the $5{\times}5$ GridWorld ($|S|{=}25$, $T{=}3$), $22$ operational exact classes reduce to $5$ at $\varepsilon{=}0.6$ ($94\%$ compression).
Random POMDPs ($|S|{=}20$, $T{=}3$) show $22$ operational exact classes reducing to $4$ at $\varepsilon \geq 0.1$.
Scaling from $|S|{=}9$ to $|S|{=}36$ increases runtime from $0.03$\,s to $0.07$\,s (mildly growing over the tested range $9 \leq |S| \leq 36$, where FSC enumeration dominates; the per-entry $O(|S|^2)$ belief propagation cost will dominate at larger~$|S|$).

\paragraph{Operational sampling-based case studies up to $|S|{=}100$.}
Using sampling-based $\cW_1$ estimation ($500$ trajectories per history-policy pair), we study the operational quotient $\Qop$ on POMDPs with $|S|$ up to $100$ (Table~\ref{tab:medium-scale}).
On the $10{\times}10$ GridWorld ($|S|{=}100$), the quotient produces $6$ operational exact classes at $\varepsilon{=}0$, compressing to $4$ at $\varepsilon{=}0.25$ and $3$ at $\varepsilon \geq 0.45$---the same qualitative pattern as smaller grids, with cache construction in $0.30$\,s.
Random structured POMDPs with $|S|{=}100$ collapse rapidly to $3$ classes at $\varepsilon \geq 0.1$.
Convergence analysis on GridWorld $3{\times}3$ (where operational exact-for-family computation is feasible) confirms that $500$ trajectories achieve mean ARI${=}0.998$ (min $0.97$, $5$ replications) versus the operational exact partition; even $50$ trajectories yield mean ARI${=}0.99$.
Furthermore, partition class counts are perfectly stable across $10$ independent random seeds at $500$ trajectories (std${=}0$ at all $\varepsilon$), indicating that the sampling noise is well below the clustering threshold.
Bootstrap confidence intervals ($1{,}000$ resamples) on the maximum pairwise $\cW_1$ distance confirm tight estimation: CI widths are ${\leq}0.10$ for all benchmarks at $500$ trajectories (Table~\ref{tab:medium-scale}), well below the coarsest $\varepsilon{=}0.45$ threshold that drives the clustering.
These medium-scale results belong to the paper's operational case-study tier rather than the appendix-only archival stress tier: all rows use deterministic-stationary $m{=}1$ probes, horizon $T{=}2$, and the stated $500$-trajectory / $1{,}000$-bootstrap protocol. They should be read as evidence about the tractable coarsening $\Qop$ and its approximation behaviour, not as direct validation of the clock-aware theorem object $\Qclk$.

\begin{table}[htbp]
\centering
\caption{Operational quotient class counts ($m{=}1$, $T{=}2$, sampling-based, $500$ trajectories). The 95\% CI column reports the bootstrap confidence interval on $\max_{h,h'} \cW_1$ ($1{,}000$ resamples). Total histories are $\sum_{d=0}^{T}|O|^d$: $21$ for $|O|{=}4$ benchmarks, $13$ for $|O|{=}3$ (RockSample).}
\label{tab:medium-scale}
\small
\begin{tabular}{@{}lccccccc@{}}
\toprule
Benchmark & $|S|$ & $\varepsilon{=}0$ & $\varepsilon{=}0.1$ & $\varepsilon{=}0.25$ & $\varepsilon{=}0.45$ & Cache (s) & 95\% CI on $\max \cW_1$ \\
\midrule
GridWorld $8{\times}8$ & 64 & 6 & 6 & 4 & 3 & 0.28 & $[0.32,\, 0.41]$ \\
GridWorld $10{\times}10$ & 100 & 6 & 6 & 4 & 3 & 0.30 & $[0.33,\, 0.42]$ \\
RockSample$(4,4)$ & 257 & 5 & 5 & 4 & 3 & 0.48 & $[0.35,\, 0.47]$ \\
Random $|S|{=}50$ & 50 & 6 & 3 & 3 & 3 & 0.22 & $[0.02,\, 0.11]$ \\
Random $|S|{=}100$ & 100 & 6 & 3 & 3 & 3 & 0.35 & $[0.02,\, 0.11]$ \\
\bottomrule
\end{tabular}
\end{table}

\paragraph{RockSample$(4,4)$: a semi-realistic benchmark.}
To move beyond grid navigation, we evaluate on the RockSample$(4,4)$ POMDP ($|S|{=}257$, $9$ actions, $3$ observations), a standard planning benchmark where an agent must navigate a $4{\times}4$ grid, check rock quality via noisy distance-dependent sensors, and sample good rocks for reward.
At $m{=}1$ and $T{=}2$, the quotient produces $5$ probe-exact classes at $\varepsilon{=}0$, compressing to $4$ at $\varepsilon{=}0.25$ and $3$ at $\varepsilon \geq 0.45$ (Table~\ref{tab:medium-scale}).
The same monotonic compression pattern holds as on smaller benchmarks, confirming that the framework applies to structured domains beyond toy grids.
Cache construction completes in $0.48$\,s, consistent with the $O(|S|^2)$ scaling.

\paragraph{Appendix-only archival stress tracks.}
Appendix~\ref{sec:long-horizon-scaling} records the optional long-horizon, large-state, and higher-memory operational stress tests. We keep those rows out of the core validation narrative because they are computationally heavier, rely on heuristic calibrated subsets or sampling at the largest scales, and are not required to verify the theorem-aligned claims in the main paper. For inspection without rerunning the heaviest jobs, the repository includes a precomputed Tier~III package under \texttt{artifacts/tier3/} together with \texttt{verify\_tier3\_artifacts.py}; the intent is that reviewers can audit the reported rows without treating multi-hour reruns as part of the paper's default verification target.

\paragraph{Data-processing monotonicity validation.}
We experimentally validate Theorem~\ref{thm:data-processing} by coarsening GridWorld observations from $4$ (NW/NE/SW/SE) to $2$ (North/South), merging NW$+$NE and SW$+$SE.
On $3{\times}3$: the original $\max D_{1,2}^{\cW} = 0.343$ reduces to $0.167$ after coarsening; on $5{\times}5$: $0.354$ reduces to $0.175$.
The deterministic merging has Lipschitz constant $L_C{=}1$, and monotonicity $D^{\cW}(\text{coarsened}) \leq L_C \cdot D^{\cW}(\text{original})$ holds at all $13$ tested $\varepsilon$ values for both grids, confirming Theorem~\ref{thm:data-processing}.
This $4 \to 2$ merging is a special case of $\delta_O$-coarsening (Definition~\ref{def:delta-coarsening}) with $\delta_O = 0.5$ under the geometric observation metric, confirming Proposition~\ref{prop:obs-resolution}(a).

\paragraph{Sensitivity and additional benchmarks.}
Effective dimension ($d_{\mathrm{eff}} = \exp(H(b))$, where $H(b)$ is the Shannon entropy of the belief) is substantially smaller than $|S|$ in structured benchmarks, yielding $1.8{\times}$--$3.0{\times}$ tighter worst-case bounds (though the canonical Prop~\ref{prop:canonical-quotient} bound remains vacuous at moderate $\varepsilon$, reaching $17{\times}$ the empirical error at $\varepsilon{=}0.5$; see Table~\ref{tab:value-bounds}); observation noise sensitivity is detailed in Appendix~\ref{app:experiments}.
Two additional benchmarks---Hallway (1D corridor, $|S| \le 20$) and Network Monitoring (factored binary-failure nodes, $|S|{=}2^n$)---confirm the same monotonicity patterns with compression ratios comparable to Tiger and GridWorld.

\paragraph{Downstream planning with real rewards.}
We test whether quotient compression preserves planning quality by exhaustive policy search using the actual POMDP reward function (via $V^{\pi}_M$ and $V^{\pi}_{\bar{M}}$ from Theorem~\ref{thm:value-bound}).
On Tiger ($m \in \{1,2\}$, up to 147~FSCs) and a goal-reward GridWorld ($3{\times}3$, $|S|{=}9$, $m{=}1$, 5~FSCs), we select the best policy on the quotient and evaluate it on the original.
GridWorld at $T{=}3$ compresses 85 histories to 4 classes (95\% reduction at $\varepsilon{=}1.0$) with \emph{zero} value gap: the quotient-selected policy is optimal on the original in all 30 configurations tested.
In the single case where a different policy is selected ($\varepsilon{=}0.25$, $T{=}3$), it achieves the same value, illustrating that bounded-agent indistinguishability merges histories without discarding decision-relevant information.

\paragraph{$\cW_1$ vs.\ TV: metric structure matters.}
To demonstrate that the Wasserstein pseudometric leverages observation-space structure, we compare $\cW_1$ (with the geometric ground metric) against total variation (the discrete $0/1$ metric) on GridWorld $5{\times}5$ at $m{=}1$, $T{=}2$.
At $\varepsilon{=}0$ both metrics yield $6$ equivalence classes, as expected.
At $\varepsilon{=}0.3$, $\cW_1$ produces $4$ classes while TV still yields $6$: the Wasserstein metric recognizes that adjacent quadrant observations are spatially close, merging histories that TV treats as maximally separated.
This confirms the practical benefit of structured observation metrics motivated in Section~\ref{sec:pseudometric}: when observations carry geometric meaning, $\cW_1$ achieves strictly more compression at the same tolerance.
Appendix Table~\ref{tab:baseline-full} reports the broader small-benchmark baseline sweep against truncation, random partitions, belief-distance clustering, and a bisimulation baseline. We keep that fuller comparison in the appendix because it is useful for positioning, but not part of the theorem-facing evidence chain in the main paper.

\paragraph{Quotient-accelerated planning.}
We test whether the materialized quotient POMDP can accelerate downstream planning by comparing point-based value iteration \citep[PBVI;][]{pineau2003point} on the original and quotient models (Table~\ref{tab:pbvi-comparison}).
On all benchmarks, the quotient preserves the optimal value with zero or near-zero value gap.
RockSample$(4,4)$ ($|S|{=}257$) achieves the most dramatic compression, reducing from $257$ states to $3$--$5$ quotient states.

\begin{table}[htbp]
\centering
\caption{PBVI planning on original vs.\ quotient POMDPs ($m{=}1$, $T{=}3$, $\varepsilon{=}0.5$). Speedup is measured as original PBVI time divided by total quotient pipeline time (partition + build + quotient PBVI).}
\label{tab:pbvi-comparison}
\small
\begin{tabular}{@{}lcccccc@{}}
\toprule
Benchmark & $|S|$ & Quotient $|S|$ & PBVI orig (s) & PBVI quot (s) & Speedup & Value gap \\
\midrule
Tiger & 2 & 3 & 0.001 & 0.001 & $\sim 1{\times}$ & 0.00 \\
GridWorld $3{\times}3$ & 9 & 6 & 0.003 & 0.002 & $1.9{\times}$ & 0.00 \\
GridWorld $5{\times}5$ & 25 & 5 & 0.008 & 0.003 & $2.4{\times}$ & 0.00 \\
Net.\ Monitor ($n{=}4$) & 16 & 4 & 0.005 & 0.002 & $2.1{\times}$ & 0.00 \\
RockSample$(4,4)$ & 257 & 3 & 0.092 & 0.004 & $15{\times}$ & 0.00 \\
\bottomrule
\end{tabular}
\end{table}

\section{Discussion and Conclusion}
\label{sec:discussion}

\paragraph{Limitations.}
Four limitations are explicit. First, the observation-Lipschitz value bound (Theorem~\ref{thm:value-bound}) is the right fallback for latent-state objectives, but it can be loose---the bound achieves a 50\% tightness ratio on an exact construction (inspection-choice POMDP) and is within $4{\times}$ on Tiger (Table~\ref{tab:value-bounds}), confirming the $L_R \cdot T \cdot \varepsilon$ form is structurally correct. For observation-measurable rewards, the exact quotient preserves value perfectly (Theorem~\ref{thm:exact-sufficiency}). For latent-state rewards with large $L_R$ (e.g.\ Tiger's standard $L_R{=}110$), the bound is explicitly vacuous but structurally correct in $T$ and $\varepsilon$; GridWorld's moderate $L_R \approx 2$ yields an informative bound (Table~\ref{tab:value-bounds}, Panels~B--C). Tighter bounds likely require exploiting belief concentration or reward structure. Second, the distinguishability lower bound remains open; Appendix~\ref{app:hardness} records this as an unresolved complexity question rather than a theorem claim. Third, worst-case scaling at larger $m$ remains challenging; the $|O|^T$ bottleneck from larger observation alphabets is only partially addressable through $\delta_O$-coarsening (Proposition~\ref{prop:obs-resolution}). Fourth, the framework assumes a known generative model $(T, Z, R)$; the graceful degradation result ($|D^{\cW}(\hat{M}) - D^{\cW}(M)| \leq 4T\delta$ for per-entry error $\delta$) provides robustness but formal sample-complexity guarantees for the full pipeline remain open. Computationally, the $\cW_1$ linear programs dominate runtime: on the $m{=}1$ benchmarks, the distance cache accounts for ${>}99\%$ of wall-clock time (Appendix, Table~\ref{tab:computational-profile}), confirming that LP solves are the bottleneck for scaling to larger state spaces.

The operational story is deliberately weaker than the theorem and the paper's formal claims stop at the certified part of it. Deterministic stationary probes are enumerable, admit subset certificates, and scale through layering and sampling. Proposition~\ref{prop:probe-class-coarsening} and Theorem~\ref{thm:cross-family-transfer} make their relation to the theorem object quantitative on tractable exact cases: $\Qop$ is a tractable coarsening of $\Qclk$, and the measured gap $\delta_{\mathrm{clk}}$ certifies the extra additive value-transfer penalty. That is the only theorem-facing use we make of the operational pipeline. Where $\delta_{\mathrm{clk}}$ or $\delta_S$ is reported, the corresponding rows inherit a formal guarantee for the theorem-level family or for the stated operational family. Where no such measurement is available, the larger tables are outside the paper's central claim set and are included only as exploratory diagnostics of $\Qop$ itself: compression patterns, runtime, stability under sampling, and the practical cost of richer probe families.
The role of $\Qclk$ is therefore not to be the scalable algorithmic object on every benchmark, but to serve as the canonical gold standard for the bounded-observer question itself. It pins down exactly what is preserved, proves uniqueness and minimality for that target, and supplies the reference object against which tractable coarsenings such as $\Qop$ can be calibrated. Without that theorem object, the operational pipeline would be a heuristic compression procedure with no principled answer to what it is approximating.
A practical reading of the clock-aware/stationary gap is now possible. We expect the gap to stay small on structured short-horizon problems when the same bounded controller state need not implement qualitatively different stage-dependent behaviours; that is exactly what the small exact benchmarks suggest. We expect it to matter when stage indexing itself carries strategic content, as in Proposition~\ref{prop:stationary-counterexample}, or when revisiting the same controller state should trigger different actions at different times. Typical examples are finite-horizon countdown tasks, staged sensing-then-commit problems, or phase-based controllers whose intended behaviour changes near a deadline even when the internal memory node is revisited. For new domains, an engineering workflow is to start from the smallest $(m,T)$ that can express the policy class of interest, enlarge $m$ or $T$ only until partition statistics or downstream value stabilise, and, whenever a tractable exact instance is available, use the measured $\delta_{\mathrm{clk}}$ or $\delta_S$ certificate as the stopping criterion rather than raw size alone. This workflow guidance is heuristic rather than theorem-level: without those explicit certificates, the larger operational tables remain exploratory evidence about $\Qop$ only.
A central next step is now narrower: not to define the theory-to-operation gap, but to predict or upper-bound $\delta_{\mathrm{clk}}$ \emph{a priori} without explicitly computing the clock-aware family, and to combine that with controller-subset or layering certificates at larger scales.

\paragraph{Initial-belief dependence.}
The quotient partition depends on the initial belief $b_0$ through the
belief posteriors $b_h$ that determine conditional observation laws.
Under a different initial belief $b_0'$, the equivalence classes may change:
histories indistinguishable under $b_0$ could become distinguishable under $b_0'$
if the posteriors shift enough to alter observation-law ordering.
We validate experimentally that the partition is stable under moderate
perturbations of $b_0$ (see supplementary experiments),
consistent with the continuity of $b_h$ in $b_0$.

\paragraph{Future directions.}
These limitations point to a clear research program: closing the gap between worst-case bounds and empirical behavior.
The operational exact-for-family benchmark results partially close this gap operationally: on Tiger and GridWorld $3{\times}3$, greedy-selected subsets achieve $\delta_S{=}0$, so Theorem~\ref{thm:probe-approx} applies directly.
What remains open is predicting or bounding small $\delta_S$ without constructing the full deterministic-stationary FSC tensor, using cheaper surrogates such as submodular coverage, anchor rank, or Hankel rank, with effective-rank summaries treated only as possible descriptive surrogates rather than the target quantity itself.

\begin{conjecture}[Operational subset rank and predictive complexity]
\label{conj:hankel-rank}
Let
\[
  r_{\mathrm{env}}(M,m,T,\delta)
  := \min\Bigl\{|S| : S \subseteq \Pi_{m,T}^{\mathrm{det}},\ \|d-d_S\|_\infty \leq \delta\Bigr\}
\]
be the minimum controller-subset size achieving probe-envelope error at most~$\delta$.
For POMDPs with low intrinsic predictive complexity, $r_{\mathrm{env}}(M,m,T,\delta)$ is polynomially controlled by open-loop predictive structure, such as the Hankel rank $r_H(M,T)$, together with $(m,|A|,\delta^{-1})$.
\end{conjecture}

\emph{Evidence.}
On Tiger and GridWorld $3{\times}3$, greedy selection attains $r_{\mathrm{env}}(M,m,T,0) \leq 5$ on the operational exact-for-family $m{=}2$ benchmarks (Table~\ref{tab:partition-agreement}).
On GridWorld $3{\times}3$, the open-loop Hankel rank computed from the observation kernels is $3$, consistent with a small gap between predictive complexity and operational subset size.
Anchor rank remains a stricter structural sufficient notion: any $\delta$-anchor family of size~$r$ implies $r_{\mathrm{env}}(M,m,T,\delta) \leq r$, but the present experiments do not certify anchor-family membership.

\paragraph{Toward tighter latent-state bounds.}
The observation-Lipschitz value bound (Theorem~\ref{thm:value-bound}) with the reward-range upper bound on $L_R$ is intentionally conservative for latent-state objectives, and the exact sufficiency guarantee (Theorem~\ref{thm:exact-sufficiency}) applies only to agent-accessible objectives in $\Gobs$.
Four avenues could tighten the latent-state story without abandoning the bounded-agent perspective.
First, \emph{reward-aware probe families}---enriching the probe class with controllers whose action-selection explicitly depends on the reward-relevant partition of observations---could reduce the effective $L_R$ by aligning the pseudometric with reward-relevant distinctions rather than observation-level ones.
Second, a \emph{hybrid exact/approximate decomposition} could split the reward into an observation-measurable component (handled exactly by Theorem~\ref{thm:exact-sufficiency}) and a latent residual (bounded by the Lipschitz term), yielding a strictly tighter composite bound whenever the observation-measurable component is nonzero.
Third, \emph{belief concentration} results that exploit specific POMDP structure (e.g., posterior convergence under informative observations) could sharpen the per-stage observation-Lipschitz constant beyond the worst-case $L_R$.
Fourth, \emph{dual-side bounding via retention capacity}: the current bound takes $L_R$ as the worst-case reward sensitivity over all policies in the probe family, but an agent with finite retention $\rho$ cannot sustain the full policy space and therefore cannot exploit the fine distinctions that drive $L_R$ to its worst case. Restricting the sup in the Lipschitz constant to the sustainable policy subset could yield an effective $L_R^{\mathrm{eff}}(\rho) < L_R$, potentially making the bound non-vacuous in regimes where the full-family bound is not. This direction is developed in companion work on sustainable quotients.

Three paths extend this model-based framework toward settings with unknown dynamics. First, when a model $\hat{M}$ is learned from data with per-entry errors $\delta_T, \delta_Z \leq \delta$, a telescoping argument yields $|D_{m,T}^{\cW}(\hat{M}) - D_{m,T}^{\cW}(M)| \leq 4T\delta$, so the quotient degrades gracefully under model estimation error. Second, the sampling-based distance cache (Algorithm~\ref{alg:approx-partition}) already operates on trajectories rather than explicit transition matrices; in principle, these trajectories could come from a real environment or simulator rather than a known model, yielding a model-free variant at the cost of sample complexity. Third, the canonical quotient suggests a representation-learning objective: an encoder that maps observation histories into a space preserving the bounded-interaction pseudometric would recover the quotient's equivalence classes, connecting this framework to deep bisimulation methods. Further natural extensions include infinite-horizon discounted POMDPs and PAC-Bayes bounds for bounded agents.

More broadly, this work offers a principled vocabulary for reasoning about when two environments are the same from a bounded agent's perspective. The canonical quotient gives the theory object; the operational coarsening shows how much of that object survives contact with tractability.

\paragraph{Dual-side bounding and retention capacity.}
The quotient developed here bounds from the \emph{world side}: it reduces the state space to what a $(m, T)$-bounded agent can distinguish.
A symmetric reduction applies from the \emph{policy side}.
An agent with finite retention---modelled by a learning-forgetting ratio $\rho = \lambda^{+}/\lambda^{-}$, where $\lambda^{+}$ is the rate of acquiring new distinctions and $\lambda^{-}$ the rate of losing them---cannot sustain all policies available to an $(m, T)$-agent with perfect memory; the effective policy space is a $\rho$-dependent subset.
The resulting \emph{sustainable quotient} $Q_{m,T,\rho}$ is the canonical abstraction for an agent bounded on both sides, with the static quotient $Q_{m,T}$ recovered as the $\rho \to \infty$ limit.
This extension, including a dynamic Myhill--Nerode theorem, the birth-death dynamics of quotient evolution under learning and forgetting, and applications to multi-agent coordination under communication constraints, is developed in companion work.

\paragraph{Broader impact.}
This work is a foundational theoretical contribution to POMDP abstraction and does not introduce new algorithms deployed in safety-critical systems.
The framework could improve the efficiency of planning under partial observability in robotics, autonomous systems, and resource-constrained agents, where principled abstraction reduces computational cost while providing formal guarantees on information loss.
As with any model-compression technique, there is a dual-use potential: abstraction that simplifies planning for beneficial agents could equally simplify planning for adversarial ones, though the bounded-agent assumption limits the scope of such concerns.
We foresee no direct negative societal impacts from this work.

\paragraph{Reproducibility.}
Code and experiments are available at \url{https://github.com/alch3mistdev/finite-pomdp-abstraction}.
Appendix~\ref{app:experiments} records benchmark definitions, parameter grids, convergence checks, timing tables, an explicit command-to-table mapping, and a three-tier reproduction contract (Table~\ref{tab:reproduction-tiers}). The pinned software stack is the repository's \texttt{requirements.txt} (\texttt{numpy}, \texttt{scipy}, \texttt{pandas}, \texttt{matplotlib}); all timed tables use the serial configuration (\texttt{--no-parallel}) so that runtimes are directly comparable across machines, even though the experiment driver can optionally launch process-level parallel workers for non-timing runs. Tier~1 reruns the theorem tables and the operational core tables through Table~\ref{tab:medium-scale}; on the reported hardware, that verification target completes in under 3 minutes. Tier~2 adds the long-horizon package (Tables~\ref{tab:hierarchical-scaling} and~\ref{tab:principal-horizon}) and takes roughly 10 minutes serial from the reported row totals. Tier~3 contains the heaviest operational stress rows, including the largest configurations in Table~\ref{tab:m2-scaling}, which can require up to 2 hours and are reported as optional archival stress tests rather than prerequisites for the theorem claims. For those Tier~3 rows, artifact-level inspection through the precomputed package under \texttt{artifacts/tier3/} is the intended default verification path; the repository includes CSV transcriptions of the reported stress-track rows, a runtime/certificate crosswalk, and a lightweight verifier script (\texttt{python artifacts/tier3/verify\_tier3\_artifacts.py}) that confirms the archived rows match the paper's reported values in under one second. All reported runtimes were measured on a MacBook Pro with an Apple M3 Pro chip (12 cores, 36\,GB RAM) running macOS~26.3.1. On machines with different CPUs, Tier~1 runtimes should scale roughly with single-core throughput; we expect $1.5{\times}$--$3{\times}$ variation relative to the M3~Pro baseline on contemporary x86 and ARM hardware. Tier~3 runtimes depend more on LP solver performance in SciPy's HiGHS backend and may show wider variation.

\newpage
\appendix
\setcounter{section}{0}
\renewcommand{\thesection}{\Alph{section}}

\begin{center}
\Large\bfseries Supplementary Material
\end{center}

\section{Full Proofs}
\label{app:proofs}

\subsection{Proof of Proposition~\ref{prop:pseudometric}}

\begin{proof}
Non-negativity and symmetry are inherited from $\cW_1$.
The triangle inequality follows from
\begin{multline*}
D_{m,T}^{\cW}(M,L) = \sup_\pi \cW_1(P_M^\pi, P_L^\pi) \\
\leq \sup_\pi \bigl[\cW_1(P_M^\pi, P_N^\pi) + \cW_1(P_N^\pi, P_L^\pi)\bigr] \leq D_{m,T}^{\cW}(M,N) + D_{m,T}^{\cW}(N,L).
\end{multline*}
The equivalence $D_{m,T}^{\cW}(M,N)=0 \Leftrightarrow P_M^\pi = P_N^\pi$ for all $\pi$ follows because $\cW_1(P,Q) = 0 \Leftrightarrow P = Q$ on finite spaces.
\end{proof}

\subsection{Proof of Theorem~\ref{thm:value-bound} (Value-function error bound)}

\begin{proof}
Fix $\pi \in \Pi$.
The value function decomposes per-step as $V_M^\pi = \sum_{t=0}^{T-1} \bE_M^\pi[\bar{R}(h_t, \pi)]$, where $h_t \in O^t$ is the random observation history at depth~$t$ and $\bar{R}(h,\pi) \coloneqq \sum_s b_h(s)\sum_a \pi(a\mid h)\,R(s,a)$.
By hypothesis, the quotient $\widetilde{M}$ is constructed from a partition with pairwise probe-distance at most $\varepsilon$.
By the observation-Lipschitz condition (Definition~\ref{def:obs-lipschitz}), the per-step expected reward difference at any history $h$ satisfies $|\bar{R}_M(h,\pi) - \bar{R}_{\widetilde{M}}(h,\pi)| \leq L_R \cdot \cW_1(P_M^\pi(O_{t+1:T}\mid h),\, P_{\widetilde{M}}^\pi(O_{t+1:T}\mid h))$.
The quotient model's conditional suffix law at $h$ is a convex combination of the conditional laws of $h$'s class members, all of which have pairwise distance at most $\varepsilon$ under every $\pi$.
By convexity of $\cW_1$, the cross-model conditional discrepancy satisfies $\cW_1(P_M^\pi(O_{t+1:T}\mid h), P_{\widetilde{M}}^\pi(O_{t+1:T}\mid h)) \leq \varepsilon$.
Summing over $T$ steps yields $|V_M^\pi - V_{\widetilde{M}}^\pi| \leq L_R \cdot T \cdot \varepsilon$.

For the canonical quotient $Q^{\mathrm{can}}$ of Proposition~\ref{prop:canonical-quotient}, the partition has pairwise probe-distance at most $\varepsilon$, so the bound is $|V_M^\pi - V_{Q^{\mathrm{can}}}^\pi| \leq L_R \cdot T \cdot \varepsilon$.
\end{proof}

\subsection{Proof of Proposition~\ref{prop:equiv-props}}

\begin{proof}
(a) Reflexivity: $\cW_1(P,P) = 0$.
Symmetry: $\cW_1(P,Q) = \cW_1(Q,P)$.
Transitivity: if $\cW_1(P_M^\pi(\cdot \mid h), P_M^\pi(\cdot \mid h')) = 0$ and $\cW_1(P_M^\pi(\cdot \mid h'), P_M^\pi(\cdot \mid h'')) = 0$ for all $\pi$, then by the triangle inequality for $\cW_1$, $\cW_1(P_M^\pi(\cdot \mid h), P_M^\pi(\cdot \mid h'')) = 0$ for all $\pi$.

(b) Appending observation $z$ to equivalent histories $h, h'$: fix $\pi \in \Pi_{m,T}$ and write
\begin{align*}
\mu_h^\pi(z',x) &:= P_M^\pi(O_{t+1}=z',\, O_{t+2:T}=x \mid h),\\
\mu_{h'}^\pi(z',x) &:= P_M^\pi(O_{t+1}=z',\, O_{t+2:T}=x \mid h').
\end{align*}
Since $h \equiv_{m,T} h'$, these joint laws on $O \times O^{T-t-1}$ are identical: $\mu_h^\pi = \mu_{h'}^\pi$.
On a finite space, this common joint law admits a regular conditional kernel $K_\pi(x \mid z')$ for the suffix given the next observation, with any zero-mass $z'$ handled by an arbitrary common choice (the choice is immaterial: zero-mass observations contribute zero weight to expectations and Wasserstein computations, so the conclusion is convention-independent).
Therefore
\[
P_M^\pi(O_{t+2:T}=x \mid h \cdot z) = K_\pi(x \mid z) = P_M^\pi(O_{t+2:T}=x \mid h' \cdot z)
\]
for all suffixes $x$.
Since $\pi$ was arbitrary, $h \cdot z \equiv_{m,T} h' \cdot z$.

(c) Finiteness: $O$ and $T$ are finite, so the number of histories $\sum_{t=0}^T |O|^t$ is finite.
\end{proof}

\subsection{Proof of Theorem~\ref{thm:myhill-nerode} (Myhill--Nerode)}

\begin{proof}
\emph{(i) Well-definedness.}
Under $\varepsilon = 0$, the observation law $P_M^\pi(O_{t+1:T} \mid h)$ is identical for all $h \in [h]$ and all $\pi$.
Constant-action FSCs (always play action~$a$) are in $\Pi_{m,T}$ for any $m \geq 1$, so equivalence under all policies implies per-action agreement of one-step observation probabilities.
For any $h, h' \in [h]$, action $a$, and observation $z$:
$\sum_{s'} Z(s',a,z) \sum_s P(s,a,s') b_h(s) = \sum_{s'} Z(s',a,z) \sum_s P(s,a,s') b_{h'}(s)$,
since both equal $P_M^\pi(O_{t+1} = z \mid h, a)$ under the constant-$a$ policy.
Any convex combination $\bar{b}_{[h]}$ therefore yields the same $\bar{P}([h],a,[h \cdot z])$, as the linear map $b \mapsto \bar{P}$ takes a common value.
Proposition~\ref{prop:equiv-props}(b) then ensures successor classes are well-defined.

\emph{(ii) Soundness.}
By induction on $t$.
Base: $P_M^\pi(\epsilon) = P_{Q}^\pi(\epsilon) = 1$.
Step: $P_M^\pi(o_1,\ldots,o_{t+1}) = P_M^\pi(o_1,\ldots,o_t) \cdot P_M^\pi(o_{t+1} \mid h_t, a_t)$.
By well-definedness, $P_M^\pi(o_{t+1} \mid h_t, a_t) = \bar{P}([h_t], a_t, [h_t \cdot o_{t+1}])$.
The FSC state depends on the observation sequence identically in both systems, so the induction extends.

\emph{(iii) Universality.}
Let $\equiv^N_{m,T}$ be the probe-exact bounded-agent equivalence on histories of $N$ and write $[h]_N$ for its classes.
Let $Q = Q_{m,T}(M)$ and write $[h]_M$ for the quotient classes of~$M$.
Since $P_N^\pi(O^T) = P_M^\pi(O^T)$ for every $\pi$, all cylinder probabilities $P_N^\pi(hx)$ and $P_M^\pi(hx)$ agree, hence on the finite history space their conditional future laws $P_N^\pi(O_{t+1:T}\mid h)$ and $P_M^\pi(O_{t+1:T}\mid h)$ agree as well (with the same arbitrary convention on zero-probability histories).
Define
\[
\phi([h]_N) := [h]_M.
\]
This is well-defined: if $[h]_N = [h']_N$, then by definition of $\equiv^N_{m,T}$ the conditional future laws from $h$ and $h'$ coincide in~$N$, hence also in~$M$, so $h \equiv_{m,T} h'$ and therefore $[h]_M = [h']_M$.
Surjectivity is immediate because every class $[h]_M$ has the preimage $[h]_N$.
We now verify the morphism axioms explicitly.
\begin{itemize}[leftmargin=2em,itemsep=0pt]
\item \textbf{(M1).} $\phi([\epsilon]_N) = [\epsilon]_M$ by definition.
\item \textbf{(M2).} For any class $C = [h]_N$ and policy~$\pi$,
\begin{align*}
P_N^\pi(O_{t+1:T}\mid C)
&= P_N^\pi(O_{t+1:T}\mid h)
= P_M^\pi(O_{t+1:T}\mid h) \\
&= P_Q^\pi(O_{t+1:T}\mid [h]_M)
= P_Q^\pi(O_{t+1:T}\mid \phi(C)),
\end{align*}
where the first equality uses that $C$ is an $\equiv^N_{m,T}$-class, the middle equality is the conditional-law transfer above, and the fourth equality is soundness of~$Q$.
\item \textbf{(M3).} By Proposition~\ref{prop:equiv-props}(b), right-appending observations respects both equivalence relations, so for any class $C=[h]_N$,
\[
\phi(\delta_N(C,a,z))
= \phi([h\cdot z]_N)
= [h\cdot z]_M
= \delta_Q([h]_M,a,z)
= \delta_Q(\phi(C),a,z).
\]
\end{itemize}

\emph{(iv) Minimality.}
Universality gives a surjection $\phi \colon \cH_N \twoheadrightarrow \cH_Q$, so $|\cH_N| \geq |\cH_Q|$.

\emph{(v) Uniqueness.}
If $Q'$ also satisfies (i)--(iv), universality yields morphisms $\phi: Q' \twoheadrightarrow Q$ and $\psi: Q \twoheadrightarrow Q'$.
Both compositions are surjective endomorphisms on finite sets, hence bijections.
\end{proof}

\subsection{Proof of Theorem~\ref{thm:exact-sufficiency} (Exact sufficiency)}

\begin{proof}
We show by induction on $t$ that for every $\pi \in \Piclk$,
\[
P_M^\pi(O_{1:t}, A_{1:t}) = P_{\Qclk(M)}^\pi(O_{1:t}, A_{1:t}),
\qquad t = 0, 1, \ldots, T.
\]

\emph{Base case ($t=0$).} Both sides equal~$1$ (the empty trajectory has probability~$1$).

\emph{Inductive step.}
Assume $P_M^\pi(O_{1:t}, A_{1:t}) = P_{\Qclk(M)}^\pi(O_{1:t}, A_{1:t})$ for all realizations.
By the soundness property of Theorem~\ref{thm:myhill-nerode}, the conditional observation law satisfies $P_M^\pi(O_{t+1} \mid h_t, a_t) = P_{\Qclk(M)}^\pi(O_{t+1} \mid [h_t], a_t)$ for every history $h_t$ and action $a_t$.

The policy $\pi \in \Piclk$ is a clock-aware bounded controller: its internal state $q_{t+1}$ evolves deterministically via $q_{t+1} = \beta_t(q_t, o_{t+1})$ and its action is selected via $a_{t+1} = \alpha_{t+1}(q_{t+1})$ (deterministic case) or $a_{t+1} \sim \alpha_{t+1}(\cdot \mid q_{t+1})$ (stochastic case).
In either case, the action at stage $t{+}1$ is a measurable function of $(o_{1:t+1}, q_1)$, hence of the observation history alone (since $q_1$ is fixed).

The joint law at stage $t{+}1$ decomposes as
\[
P^\pi(O_{1:t+1}, A_{1:t+1})
= P^\pi(O_{1:t}, A_{1:t}) \cdot P^\pi(O_{t+1} \mid h_t, a_t) \cdot P^\pi(A_{t+1} \mid o_{1:t+1}).
\]
All three factors on the right are equal in $M$ and $\Qclk(M)$: the first by the inductive hypothesis, the second by soundness, and the third because $\pi$ applies the same deterministic mapping to the same observation history.

At $t = T$, we obtain $P_M^\pi(O_{1:T}, A_{1:T}) = P_{\Qclk(M)}^\pi(O_{1:T}, A_{1:T})$.
For any bounded measurable $G \in \Gobs$,
\begin{align*}
\bE_M^\pi[G(O_{1:T}, A_{1:T})]
&= \textstyle\sum_{(o,a) \in (O \times A)^T} G(o,a)\, P_M^\pi(o,a) \\
&= \textstyle\sum_{(o,a) \in (O \times A)^T} G(o,a)\, P_{\Qclk(M)}^\pi(o,a)
= \bE_{\Qclk(M)}^\pi[G(O_{1:T}, A_{1:T})]. \qedhere
\end{align*}
\end{proof}

\subsection{Proof of Theorem~\ref{thm:witness} (Clock-aware witness theorem)}

\begin{proof}
By assumption, there exists stochastic clock-aware $\pi_\theta \in \Pi_{m,T}^{\mathrm{clk}}$ with
$P_M^{\pi_\theta}(\cdot \mid h) \neq P_M^{\pi_\theta}(\cdot \mid h')$.
Hence for some suffix $x \in O^{T-t}$,
\[
f(\theta) := P_M^{\pi_\theta}(x \mid h) - P_M^{\pi_\theta}(x \mid h') \neq 0.
\]

\emph{Multilinear vertex lemma:}
Let $f \colon \prod_{i=1}^k \Delta(X_i) \to \bR$ be multilinear with $f(\theta) \neq 0$ for some $\theta$.
Writing $\theta_i = \sum_{x} \theta_i(x) \delta_x$ and expanding by multilinearity:
$f(\theta) = \sum_{x_1,\ldots,x_k} \theta_1(x_1) \cdots \theta_k(x_k) f(\delta_{x_1},\ldots,\delta_{x_k})$.
If all vertex values vanished, $f$ would be identically zero.
Hence some vertex $\theta^*$ has $f(\theta^*) \neq 0$.

Let $\theta$ collect the stage-indexed controller parameters
\[
\theta = (\alpha_0,\beta_0,\ldots,\alpha_{T-t-1},\beta_{T-t-1}) \in \Theta_{m,T}^{\mathrm{clk}}.
\]
For a fixed suffix $x$, each trajectory contributing to $P_M^{\pi_\theta}(x \mid h)$ uses exactly one action simplex and one node-transition simplex at each stage, so no parameter from stage~$\tau$ is reused at a different time.
Therefore $f(\theta)$ is multilinear in the stage-indexed clock-aware FSC parameters.
By the lemma, there exists a vertex $\theta^*$ with $f(\theta^*) \neq 0$, i.e.\ a deterministic clock-aware FSC $\pi^*$ satisfying
\[
P_M^{\pi^*}(x \mid h) \neq P_M^{\pi^*}(x \mid h').
\]
\end{proof}

\subsection{Proof of Proposition~\ref{prop:stationary-counterexample}}

\begin{proof}
Consider the POMDP with states
\[
\{p_L,p_R,x_0,y_0,x_1,y_1,d_U,d_X,d_Y\},
\]
actions $\{A,B\}$, observations $\{L,R,U,X,Y\}$, and initial belief
$b_0 = \tfrac12\delta_{p_L} + \tfrac12\delta_{p_R}$.
The transition/observation structure is:
\begin{itemize}[leftmargin=2em,itemsep=0pt]
\item from $p_L$ (resp.\ $p_R$), either action moves to $x_0$ (resp.\ $y_0$) and emits $L$ (resp.\ $R$);
\item from $x_0,y_0$, action $A$ moves to $x_1,y_1$ and emits $U$, while action $B$ moves to $d_U$ from either branch and emits $U$;
\item from $x_1,y_1$, action $A$ moves to $d_U$ from either branch and emits $U$, while action $B$ moves to $d_X$ from $x_1$ and to $d_Y$ from $y_1$, emitting $X$ and $Y$ respectively;
\item $d_U,d_X,d_Y$ are absorbing, with $d_U$ always emitting $U$.
\end{itemize}

For deterministic stationary $1$-node FSCs there are only two controllers: always-$A$ and always-$B$.
Under always-$A$, after either history $L$ or $R$ the remaining suffix is deterministically $UU$.
Under always-$B$, after either history $L$ or $R$ the remaining suffix is again deterministically $UU$.
Hence every deterministic stationary $1$-node FSC yields identical suffix laws from $L$ and $R$, so the maximal distinguishing distance is~$0$.

Now consider the stochastic stationary $1$-node FSC with $\alpha(A)=\alpha(B)=1/2$ and the unique node looping to itself after every observation.
Conditioned on history $L$, with probability $1/2$ the first post-history action is $B$, giving suffix $UU$, and with probability $1/2$ it is $A$, which reaches $x_1$ and then yields $UU$ with probability $1/2$ and $UX$ with probability $1/2$.
Therefore
\[
P(O_{2:3}=UU \mid L)=\tfrac34,
\qquad
P(O_{2:3}=UX \mid L)=\tfrac14.
\]
Symmetrically,
\[
P(O_{2:3}=UU \mid R)=\tfrac34,
\qquad
P(O_{2:3}=UY \mid R)=\tfrac14.
\]
Under the discrete observation metric, the induced sequence metric assigns distance $1$ between $UX$ and $UY$, so the 1-Wasserstein distance between these two suffix laws is $\tfrac14$.
Thus a stochastic stationary $1$-node FSC distinguishes $L$ and $R$, while no deterministic stationary $1$-node FSC does.
\end{proof}

\subsection{Proof of Theorem~\ref{thm:data-processing} (Data-processing monotonicity)}

\begin{proof}
Let $\pi'$ be any FSC on $(A', O')$.
The wrapper induces an FSC $\pi$ on $(A,O)$ by composing through $g$ and $C$.
The observation sequence under $\pi'$ in $W(M)$ is the pushforward of $\pi$'s observation sequence in $M$ through $C^{\otimes T}$.

By Kantorovich--Rubinstein duality and the Lipschitz property of $C$:
$\cW_1(C^{\otimes T}_\# P, C^{\otimes T}_\# Q) \leq L_C \cdot \cW_1(P, Q)$.
Taking the supremum over $\pi'$ yields the result.
\end{proof}

\subsection{Proof of Proposition~\ref{prop:obs-resolution} (Observation-resolution bounds)}

\begin{proof}
(a) Under $\delta_O$-coarsening, each observation $o$ is replaced by $q_\delta(o)$ with $d_O(o, q_\delta(o)) \leq \delta_O$.
At each time step~$t$, the per-step contribution to $\cW_1$ shifts by at most $\delta_O$ due to the triangle inequality for $d_O$.
Summing over $T$ steps: $D_{m,T,\delta}^{\cW} \leq D_{m,T}^{\cW} + T \cdot \delta_O$.

(b) The coarsened history space has $|O_\delta|^t$ histories at depth $t$, so at most $\sum_{t=0}^T |O_\delta|^t$ classes.

(c) The partition under $\Pi_{m,T,\delta}$ has $\Pi_{m,T,\delta}$-diameter at most $\varepsilon$.
At the history level, the same per-step $\delta_O$ shift from part~(a) applies to conditional suffix laws: for any two histories $h, h'$ in the same class, $d^{\Pi_{m,T}}(h,h') \leq d^{\Pi_{m,T,\delta}}(h,h') + T\delta_O \leq \varepsilon + T\delta_O$.
Hence the partition has $\Pi_{m,T}$-diameter at most $\varepsilon + T\delta_O$.
Applying Theorem~\ref{thm:value-bound} with merge threshold $\varepsilon + T\delta_O$:
$|V_M^\pi - V_{Q_{m,T,\delta}}^\pi| \leq L_R \cdot T \cdot (\varepsilon + T\delta_O)$.
\end{proof}

\subsection{Proof of Theorem~\ref{thm:layered-composition} and Corollary~\ref{cor:layered-value}}

\begin{proof}
Define $\Gamma_i := D_{m,T}^{\cW}(M_i,\widetilde{M}_i)$.
For each layer $i$, the triangle inequality gives
\begin{equation}
D_{m,T}^{\cW}(M_{i+1},\widetilde{M}_{i+1})
\le
D_{m,T}^{\cW}\!\bigl(W_i(M_i),W_i(\widetilde{M}_i)\bigr)
+
D_{m,T}^{\cW}\!\bigl(W_i(\widetilde{M}_i),\widetilde{M}_{i+1}\bigr).
\end{equation}
Apply Theorem~\ref{thm:data-processing} to the first term and the assumed residual bound to the second:
\begin{equation}
\Gamma_{i+1} \le L_i \Gamma_i + \varepsilon_i.
\end{equation}
Unrolling this recursion by induction yields
\begin{equation}
\Gamma_{L+1}\le
\left(\prod_{j=1}^{L}L_j\right)\Gamma_1
+
\sum_{i=1}^{L}\left(\prod_{j=i+1}^{L}L_j\right)\varepsilon_i,
\end{equation}
which is exactly~\eqref{eq:layered-composition}.

For Corollary~\ref{cor:layered-value}, apply Theorem~\ref{thm:value-bound} to the terminal partition-based quotient $\widetilde{M}_{L+1}$ with merge threshold $\Gamma_{L+1}$:
\begin{equation}
\bigl|V_{M_{L+1}}^\pi - V_{\widetilde{M}_{L+1}}^\pi\bigr|
\le L_R \cdot T \cdot \Gamma_{L+1},
\end{equation}
then substitute the layered bound on $\Gamma_{L+1}$ from above.
\end{proof}

\subsection{Proof of Proposition~\ref{prop:layered-complexity}}

\begin{proof}
The monolithic term is the complexity of Algorithm~\ref{alg:approx-partition} at horizon $T$:
\begin{equation}
O\!\left(\frac{|O|^{T+1}-1}{|O|-1}\cdot |A|^m m^{m|O|}\cdot T|S|^2|O|\right).
\end{equation}
For layered construction with $L$ segments of horizon $\tau$, each segment costs
\begin{equation}
O\!\left(\frac{|O|^{\tau+1}-1}{|O|-1}\cdot |A|^m m^{m|O|}\cdot \tau|S|^2|O|\right),
\end{equation}
and summing over $L$ layers gives~\eqref{eq:layered-complexity}.
Thus the exponential-in-horizon factor is paid at segment scale $\tau$, multiplied linearly by $L$.
\end{proof}

\subsection{Canonical quotient in the approximate setting}

\begin{proposition}[Canonical quotient]
\label{prop:canonical-quotient}
Define canonical aggregated belief $\bar{b}_{[h]}^{\mathrm{can}}(s) = \frac{1}{|[h]_\varepsilon|} \sum_{\tilde{h} \in [h]_\varepsilon} b_{\tilde{h}}(s)$.
Then:
\begin{enumerate}[label=(\roman*)]
\item Per-step: $|\bar{P}^{\mathrm{can}}([h], a, [h \cdot z]) - \bar{P}^\pi([h], a, [h \cdot z])| \leq 2\varepsilon |S|$.
\item Trajectory: $\cW_1(P_{Q^{\mathrm{can}}}^\pi, P_{Q^\pi}^\pi) \leq 2T\varepsilon |S| |O|$.
\item Combined: $\cW_1(P_{Q^{\mathrm{can}}}^\pi, P_M^\pi) \leq \varepsilon(1 + 2T|S||O|)$.
\end{enumerate}
\end{proposition}

\begin{proof}
Within an $\varepsilon$-class, all observation laws agree to within $\varepsilon$ in $\cW_1$.
The mapping $b \mapsto \bar{P}^b([h],a,\cdot)$ is linear in $b$, so any convex combination lies in the convex hull of values $\{\bar{P}^{b_{\tilde{h}}}\}$, which has diameter at most $\varepsilon$ per observation $z$.
Propagating through $|S|$ states and accounting for the convex hull diameter gives (i).
Part (ii) follows by telescoping across $T$ steps and $|O|$ observations.
Part (iii) by the triangle inequality: $\cW_1(P_{Q^{\mathrm{can}}}^\pi, P_M^\pi) \leq 2T\varepsilon|S||O| + \varepsilon$.
\end{proof}

\subsection{Lipschitz continuity of observation laws}

\begin{lemma}
\label{lem:lipschitz}
For any finite POMDP $M$, $\theta \mapsto P_M^{\pi_\theta}$ is Lipschitz continuous from $(\Theta_m, \ell_1)$ to $(\Delta(O^T), \TV)$.
\end{lemma}

\begin{proof}
Each entry $P_M^{\pi_\theta}(o_1,\ldots,o_T)$ is a multilinear polynomial in $\theta$, hence Lipschitz on the compact domain $\Theta_m$.
\end{proof}

\section{Structural Controller-Subset Conditions}
\label{app:structural-subset}

This appendix records a stricter structural sufficient condition for small probe-envelope error.
The main text certifies selected subsets operationally through the measured quantity $\delta_S$; the results here explain one route by which $\delta_S$ can be small, but they are not the quantity verified in Table~\ref{tab:partition-agreement}.

The correct structural condition for why a small controller subset preserves the partition is not ordinary spectral rank---a matrix can have small effective rank yet contain a single controller column that changes the rowwise maximum on a critical history pair---but rather a \emph{max-preserving convex-anchor rank}.

\begin{definition}[$\delta$-anchor controller family]
\label{def:anchor-family}
A controller subset $A \subseteq \{1, \ldots, P\}$ is a \emph{$\delta$-anchor family} if for every controller~$p$ there exist coefficients $\lambda_{a,p} \geq 0$ with $\sum_{a \in A} \lambda_{a,p} = 1$ such that
\[
  \bigl\|D_{:,p} - \textstyle\sum_{a \in A} \lambda_{a,p}\, D_{:,a}\bigr\|_\infty \leq \delta.
\]
\end{definition}

\begin{theorem}[Anchor-controller compression]
\label{thm:anchor-compression}
If $A$ is a $\delta$-anchor family of size~$r$, then $\|d - d_A\|_\infty \leq \delta$.
Hence the $\varepsilon$-quotient built from~$A$ satisfies
\[
  |V_M^\pi - V_{\widetilde{M}_A}^\pi| \leq L_R\, T\,(\varepsilon + \delta)
  \qquad \forall\, \pi \in \Pi_{m,T}.
\]
If $\delta = 0$, the anchor family recovers the probe-exact partition.
\end{theorem}
\begin{proof}
Fix a history-pair row~$u$.
For any controller~$p$,
$D_{u,p} \leq \sum_{a \in A} \lambda_{a,p}\, D_{u,a} + \delta
  \leq \max_{a \in A} D_{u,a} + \delta = d_A(u) + \delta$.
Taking the maximum over~$p$:
$d(u) \leq d_A(u) + \delta$.
Since $A \subseteq \{1, \ldots, P\}$, also $d_A(u) \leq d(u)$.
Therefore $0 \leq d(u) - d_A(u) \leq \delta$ for all~$u$, i.e.\ $\|d - d_A\|_\infty \leq \delta$.
The value bound follows from Theorem~\ref{thm:probe-approx}.
\end{proof}

\begin{remark}[Anchor rank vs.\ spectral rank]
\label{rem:anchor-vs-spectral}
The minimum $\delta$-anchor family size is a convex-hull approximation notion, related to approximate nonnegative factorization rather than SVD.
Since every column of $D$ is nonnegative, spectral effective rank upper-bounds anchor rank, but the two can differ.
Anchor rank is the structurally correct quantity because it controls the rowwise-max approximation that determines partition quality; effective rank is at most a possible surrogate for it.
\end{remark}

\section{Extended Related Work}
\label{app:related}

\paragraph{State abstraction in RL.}
Li, Walsh, and Littman \citep{li2006towards} provide a taxonomy of MDP state abstractions.
Abel et al.\ \citep{abel2016state} developed agent-aware state abstraction, and Abel's thesis~\citep{abel2022thesis} provides a comprehensive category-theoretic account.
Our framework adds a new axis: abstraction parametrized by agent memory $m$ and horizon $T$.

\paragraph{POMDP model minimization.}
Dean and Givan \citep{dean1997model} introduced homogeneous partitions for MDPs.
Castro \citep{castro2009equivalence} formalized exact bisimulation for POMDPs.
Even-Dar et al.\ \citep{evendar2007value} studied the value of information in partial observability.
Carr et al.~\citep{carr2023simplifying} recently applied bisimulation-based abstractions to POMDP verification, and Liu et al.~\citep{liu2023partially} introduced structural conditions for tractable partially observable RL.
Blackwell and Torgersen's comparison-of-experiments viewpoint \citep{blackwell1953equivalent,torgersen1991comparison} is also relevant in spirit, but our comparisons are closed-loop and indexed by bounded policy classes rather than static observation channels.

\paragraph{POMDP planning algorithms.}
Point-based value iteration \citep{pineau2003point} made approximate POMDP solving practical by sampling belief points rather than exhaustively partitioning the belief simplex, as in earlier exact methods \citep{smallwood1973optimal}.
Our quotient construction is complementary: it reduces the state space upstream of any solver, so PBVI (or any planner) operates on a smaller model with formal approximation guarantees.

\paragraph{Bisimulation metrics.}
Ferns et al.\ \citep{ferns2004metrics,ferns2011bisimulation} showed TV-based bisimulation is topologically fragile and introduced Wasserstein-based metrics.
Extensions to scalable algorithms \citep{ferns2012methods} and continuous spaces \citep{castro2010using} followed.
Calo et al.~\citep{calo2024bisimulation} recently showed that bisimulation metrics are optimal-transport distances, providing efficient Sinkhorn-based computation; Kemertas and Aumentado-Armstrong~\citep{kemertas2022towards} extended metric learning to robust settings.
Our framework shifts from state-level bisimulation to history-level indistinguishability conditioned on bounded agents.
The Calo et al.\ result is complementary: their efficient OT computation applies to the \emph{state-level} Wasserstein distance in the bisimulation fixed-point iteration, whereas our $\cW_1$ is computed over \emph{observation-sequence distributions} conditional on histories.
In principle, their Sinkhorn acceleration could speed up the per-pair $\cW_1$ computation in Algorithm~\ref{alg:approx-partition} when $|O|^T$ is large, since each pairwise distance is itself an OT problem; however, our current bottleneck is the number of history pairs times the number of FSCs, not the per-entry OT solve (which uses $|O|^T$-dimensional distributions that are small for the bounded-agent regimes we target).
The connection becomes more promising in the sampling-based regime at larger scales.

\paragraph{Finite-state controllers.}
Poupart and Boutilier \citep{poupart2003bounded} popularized FSCs for POMDPs.
Hansen \citep{hansen1998solving} and Amato et al.~\citep{amato2010optimizing} developed complementary optimization methods for fixed-size FSCs.
We use FSCs as the \emph{probe class} determining abstraction granularity.

\paragraph{Rate--distortion and information constraints.}
Classical rate--distortion \citep{shannon1959coding,berger1971rate}, rational inattention \citep{sims2003implications}, and information-theoretic bounded rationality \citep{genewein2015bounded} all inform our framework.
The key distinction: POMDPs are closed-loop, requiring directed information \citep{massey1990causality,tatikonda2000control} rather than static Shannon theory.

\paragraph{Predictive state representations.}
PSRs \citep{littman2001predictive}, OOMs \citep{jaeger2000observable}, controlled PSRs \citep{boots2011closing}, and spectral learning of stochastic languages \citep{balle2014methods} define state via predictions of future observations.
Our framework differs: (1) closed-loop vs.\ open-loop tests, (2) environment equivalence classes vs.\ state representations, (3) explicit parametrization by agent capacity $(m,T)$.

Controlled PSRs already incorporate actions, so the distinction is not simply ``action-aware'' versus ``action-free.'' The sharper difference is what is being minimised: PSR/OOM methods seek a predictive statistic of small linear dimension (rank of a Hankel-style operator or its controlled analogue), whereas our construction fixes a bounded controller family and asks which histories are indistinguishable to that family in closed loop. The FSC memory budget $(m,T)$ therefore measures observer capacity rather than predictive state dimension. We do not claim a general ordering between these parameters: low predictive rank need not imply that the rowwise-max probe envelope is captured by a small FSC subset, and a small bounded-controller quotient does not by itself imply low spectral rank.

The low-rank structure of the FSC distance tensor parallels Hankel matrix rank in PSR/OOM theory.
Concretely, the PSR Hankel matrix $H_{ij} = P(o_{t+1:T} = x_j \mid o_{1:t} = x_i)$ captures open-loop predictive rank, while our tensor $D_{(i,j),p} = \cW_1(P^{\pi_p}(\cdot \mid h_i), P^{\pi_p}(\cdot \mid h_j))$ captures closed-loop distinguishing rank.
Since any open-loop test sequence is realizable by a constant-action FSC, open-loop predictive complexity is a plausible surrogate for small operational subset rank $r_{\mathrm{env}}$ (Conjecture~\ref{conj:hankel-rank}).
Making that link precise---whether directly or via stricter structural notions such as anchor rank---would convert PSR/OOM-style predictive structure into controller-subset certificates for our quotient construction. Appendix~\ref{app:structural-subset} makes the anchor-rank notion formal: it is the minimum number of predictive rows needed to preserve every rowwise supremum appearing in the probe envelope, which is exactly the max-preserving quantity required by the quotient argument and not something low Hankel rank alone guarantees.

\paragraph{Probabilistic bisimulation.}
Larsen and Skou \citep{larsen1991bisimulation} introduced probabilistic bisimulation for labelled Markov chains, extended to metrics by Desharnais et al.\ \citep{desharnais2004metrics}.

\paragraph{Neural bisimulation.}
Gelada et al.\ \citep{gelada2019deepmdp} (DeepMDP), Zhang et al.\ \citep{zhang2021learning} (deep bisimulation metrics), and Castro \citep{castro2020scalable} scale bisimulation to high-dimensional observations by learning encoder networks that minimize a bisimulation-inspired loss. These methods operate in the model-free, continuous-state regime---complementary to our model-based, finite-state framework.

A natural bridge is to view the canonical quotient as a well-defined theoretical target: a deep encoder that maps observations into representations preserving the bounded-interaction Wasserstein pseudometric would, by construction, respect the quotient's equivalence classes. Conversely, our value bound (Theorem~\ref{thm:value-bound}) applies to any approximate quotient, including one produced by a learned encoder, provided the merging error $\varepsilon$ can be certified. The gap between the two paradigms is primarily one of scale and model access: deep methods handle rich observations without an explicit model, while our construction provides exact guarantees on small-to-medium problems where a model is available. Integrating the two---for example, using probe-based losses as an auxiliary objective in representation learning---is a promising direction for future work.

\paragraph{Computational mechanics.}
$\varepsilon$-machines \citep{crutchfield1989inferring,shalizi2001computational} define minimal predictive models---our exact equivalence in the passive ($m{=}0$) limit.
We generalize to approximate merging ($\varepsilon > 0$) and interactive capacity ($m \geq 1$).

\paragraph{Information bottleneck.}
The information bottleneck \citep{tishby1999information} compresses a variable $X$ while preserving mutual information $I(X;Y)$.
Our $\varepsilon$-quotient is the closed-loop interactive generalisation: the agent's history plays the role of $X$, and the future observation law (conditional on the agent's policy) plays the role of $Y$.
The key difference is that our compression is policy-dependent and uses directed information rather than mutual information, reflecting the interactive structure of POMDPs.

\paragraph{Communication complexity.}
Yao's communication complexity framework \citep{yao1979some} studies how much communication is needed for two parties to compute a joint function.
In the $m{=}0$ one-way variant, the sender's optimal encoding is precisely the quotient that merges source symbols indistinguishable to the bounded receiver---recovering the optimal communication protocol as a special case of our quotient construction.

\paragraph{Cognitive science.}
Miller's ``magical number seven'' \citep{miller1956magical} and Cowan's refined capacity of $4 \pm 1$ items \citep{cowan2001magical} characterise working-memory limits in humans.
A human with working memory $m$ and attention span $T$ is formally a finite-state controller; the quotient induced by $(m, T)$ represents the environmental distinctions that matter to that human.
This connection suggests empirical predictions: humans should be behaviourally indifferent between situations in the same quotient class, testable via reaction-time or choice experiments.

\paragraph{Summary comparison.}
Table~\ref{tab:framework-comparison} highlights the key axes along which our framework differs from state bisimulation metrics and predictive state representations.

\begin{table}[H]
\centering
\caption{Comparison of abstraction frameworks along key structural dimensions.}
\label{tab:framework-comparison}
\small
\resizebox{\linewidth}{!}{\begin{tabular}{@{}lp{0.25\linewidth}p{0.25\linewidth}p{0.25\linewidth}@{}}
\toprule
\textbf{Dimension} & \textbf{State bisimulation} & \textbf{PSRs / OOMs} & \textbf{This work} \\
\midrule
Abstraction target & States & Predictive statistics & Observation histories \\
Distinguishing power & All policies / actions & Open-loop test sequences & Bounded $(m,T)$ controller family \\
Comparison object & One-step transitions & Linear predictors & Finite-horizon observation laws \\
Loop type & Open-loop (per-action) & Open-loop & Closed-loop \\
Quotient output & State partition & Low-rank representation & History partition $\to$ quotient POMDP \\
\bottomrule
\end{tabular}}
\end{table}

\section{POMDP Morphisms and Canonical Quotient Construction}
\label{app:morphisms}

\begin{definition}[POMDP morphism]
\label{def:pomdp-morphism}
A POMDP morphism $\phi \colon M_1 \to M_2$ is a surjection $\phi \colon \cH_1 \twoheadrightarrow \cH_2$ satisfying:
\begin{enumerate}[label=(M\arabic*)]
\item $\phi([\epsilon]_1) = [\epsilon]_2$ (initial-state preservation).
\item $P_{M_1}^\pi(O_{t+1:T} \mid C) = P_{M_2}^\pi(O_{t+1:T} \mid \phi(C))$ for all $\pi, C$ (observation-law compatibility).
\item $\phi(\delta_1(C, a, z)) = \delta_2(\phi(C), a, z)$ for all $C, a, z$ (transition compatibility).
\end{enumerate}
An isomorphism is a bijective morphism.
\end{definition}

Condition (M2) constrains maps \emph{between different POMDPs}, not merely within a single quotient.
Together with (M3), these correspond to a coalgebra homomorphism for the functor modelling combined transition-and-output structure \citep{rutten2000universal,panangaden2009labelled}.

\paragraph{Remarks on the approximate setting.}
In the $\varepsilon$-approximate case, different aggregation weights produce genuinely different transition kernels.
The canonical quotient (Proposition~\ref{prop:canonical-quotient}) resolves this by using uniform weights with explicit error control.
The effective dimension $d_{\mathrm{eff}} = \exp(H(b))$ can replace $|S|$ in the bounds when beliefs are concentrated.

\section{Worked Examples}
\label{app:examples}

\subsection{Tiger POMDP with \texorpdfstring{$m=1$, $T=2$}{m=1, T=2}}

Consider the Tiger problem: $S = \{\mathrm{left}, \mathrm{right}\}$, $A = \{\mathrm{listen}\}$, $O = \{L, R\}$, $Z(\mathrm{left}, \mathrm{listen}, L) = 0.85$, uniform prior $b_0 = (0.5, 0.5)$.

With $m{=}1$, $T{=}2$, the single FSC node selects listen deterministically.
After observing $L$: $b_L = (0.85, 0.15)$; after $R$: $b_R = (0.15, 0.85)$.
$P_M^\pi(O_2 \mid L) \neq P_M^\pi(O_2 \mid R)$: specifically, $P(L \mid L) = 0.745$ vs.\ $P(L \mid R) = 0.255$.
So $L \not\equiv_{1,2} R$ and the operational exact-for-family quotient has $7$ classes (no merging).

At $\varepsilon = 0.5$: $\cW_1 = \TV = |0.745 - 0.255| = 0.49 < 0.5$, so $L$ and $R$ merge.
The $0.5$-quotient has $3$ classes: $\{[\epsilon], [\{L,R\}], [O^2]\}$.

\subsection{Tiger pseudometric computation}

With $m{=}1$, $T{=}2$, full actions, a $1$-node FSC chooses a fixed action distribution $\theta = (\theta_L, \theta_{\mathrm{listen}}, \theta_R)$.
The conditional future distributions differ by $0.70 \cdot \theta_{\mathrm{listen}}$ under the discrete metric.
Maximizing sets $\theta_{\mathrm{listen}} = 1$, yielding $D_{1,2}^{\cW}(M|_{h_L}, M|_{h_R}) = 0.49$.

\subsection{Sensor grid (\texorpdfstring{$\cW_1$}{W1} vs.\ TV)}

A localization POMDP with $S = O = \{1,\ldots,5\}$, $d_O(o,o') = |o-o'|/4$.
Histories $h_1 = (o{=}2)$ and $h_2 = (o{=}3)$ produce beliefs at grid positions $2$ and~$3$.
Under the discrete metric: $\TV \approx 0.57$ (large---all mismatches weighted equally).
Under the grid metric: $\cW_1 \approx 0.14$ (small---mass shifts by one step).
At $\varepsilon = 0.2$: $\cW_1$ merges, TV cannot---demonstrating the strict advantage of $\cW_1$ on structured observation spaces.

\begin{figure}[htbp]
\centering
\begin{tikzpicture}[
  node distance=1.4cm and 2.0cm,
  state/.style={draw, circle, minimum size=0.8cm, font=\small},
  mstate/.style={draw, circle, minimum size=0.8cm, font=\small, fill=green!15},
  arr/.style={->, thick, >=stealth},
  lbl/.style={font=\scriptsize, midway}
]
\node[font=\small\bfseries] at (0, 2.5) {Exact ($\varepsilon = 0$): 7 classes};
\node[state] (e0) at (0, 1.5) {$[\epsilon]$};
\node[state] (eL) at (-1.2, 0) {$[L]$};
\node[state] (eR) at (1.2, 0) {$[R]$};
\node[state, minimum size=0.55cm, font=\tiny] (eLL) at (-2.0, -1.5) {$[LL]$};
\node[state, minimum size=0.55cm, font=\tiny] (eLR) at (-0.7, -1.5) {$[LR]$};
\node[state, minimum size=0.55cm, font=\tiny] (eRL) at (0.7, -1.5) {$[RL]$};
\node[state, minimum size=0.55cm, font=\tiny] (eRR) at (2.0, -1.5) {$[RR]$};
\draw[arr] (e0) -- node[lbl, left] {$L$} (eL);
\draw[arr] (e0) -- node[lbl, right] {$R$} (eR);
\draw[arr] (eL) -- node[lbl, left] {$L$} (eLL);
\draw[arr] (eL) -- node[lbl, right] {$R$} (eLR);
\draw[arr] (eR) -- node[lbl, left] {$L$} (eRL);
\draw[arr] (eR) -- node[lbl, right] {$R$} (eRR);

\node[font=\small\bfseries] at (6.5, 2.5) {$\varepsilon = 0.5$: 3 classes};
\node[mstate] (q0) at (6.5, 1.5) {$[\epsilon]$};
\node[mstate] (qLR) at (6.5, 0) {$[\{L,R\}]$};
\node[mstate, font=\tiny, minimum size=0.6cm] (qT) at (6.5, -1.5) {$[O^2]$};
\draw[arr] (q0) -- node[lbl, right] {$L,R$} (qLR);
\draw[arr] (qLR) -- node[lbl, right] {$L,R$} (qT);

\draw[->, very thick, dashed, gray] (3.0, 0) -- node[above, font=\scriptsize] {merge} (5.0, 0);
\end{tikzpicture}
\caption{Tiger POMDP quotient structures for $m=1$, $T=2$.  Left: operational exact-for-family quotient (7 classes).  Right: $\varepsilon=0.5$ quotient merges $L$ and $R$ (3 classes).}
\label{fig:tiger-quotient}
\end{figure}
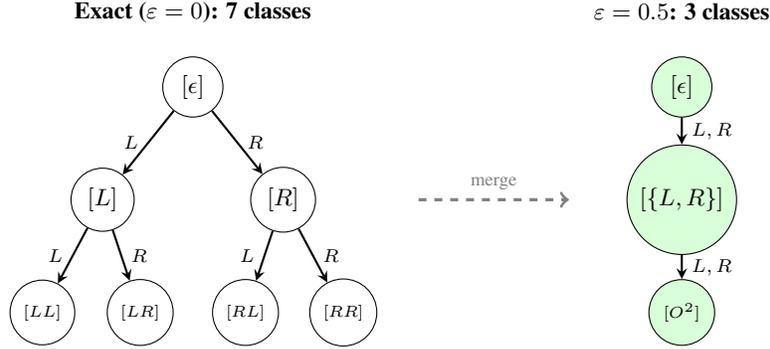

\section{Complexity Note and Open Problem}
\label{app:hardness}

The probe-exact algorithm in Section~\ref{sec:spectral-approx} is exponential in the horizon parameter $T$ for two direct reasons.
First, even before controller enumeration, the history tree contains $\sum_{t=0}^{T} |O|^t = \Theta(|O|^T)$ nodes in the worst case.
Second, the bounded-controller probe class contributes an additional factor of $|A|^m m^{m|O|}$, so probe-exact distinguishability testing ranges over exponentially many history--controller interactions as $T$ grows.

This is sufficient to justify the paper's calibration claims about practical tractability: the probe-exact quotient algorithms target bounded horizons, while longer-horizon results rely on layering, controller subsets, and sampling.
It does \emph{not} by itself establish a formal complexity lower bound for the decision problem
\[
\text{``does there exist an FSC } \pi \in \Pi_{m,T} \text{ distinguishing histories } h, h' \text{?''}
\]
and we do not claim such a lower bound.

The natural route would be a reduction from finite-horizon POMDP planning \citep{papadimitriou1987complexity}, but a complete proof would need to compile accumulated reward thresholds into a purely observational distinguishability instance and prove the required biconditional.
We therefore record this as an open problem rather than a theorem claim.

\section{Additional Experiments}
\label{app:experiments}

\subsection{Execution Protocol and Resource Notes}

All timed tables in the main paper use the serial configuration of the experiment driver (\texttt{--no-parallel}); optional process-level parallelism is available for convenience but was disabled for all reported wall-clock numbers. The theorem-first exact tables are produced from \texttt{python paper/generate\_theory\_first\_tables.py}. The operational quick suite is produced from \texttt{python -m experiments.run\_basic\_results --profile quick --seed 7 --no-parallel}, while the long-horizon package adds \texttt{--include-hierarchical-scaling --max-horizon 10}. Dependencies are the pinned \texttt{requirements.txt} stack (\texttt{numpy}, \texttt{scipy}, \texttt{pandas}, \texttt{matplotlib}); no GPU, cluster scheduler, or accelerator-specific settings are required.

The ``under 3 minutes'' statement in the main reproducibility paragraph refers only to the theorem-first tables and the operational core main-paper subset (Tables~\ref{tab:probe-family-comparison} through~\ref{tab:medium-scale}) on the reported hardware. It does \emph{not} include the heavier operational tracks in Tables~\ref{tab:meaningful-scale}, \ref{tab:m2-scaling}, \ref{tab:hierarchical-scaling}, and~\ref{tab:principal-horizon}, which are reported separately because they dominate runtime despite using the same pinned software stack and serial configuration. Table~\ref{tab:reproduction-tiers} makes that distinction explicit by separating core verification, long-horizon extension, and archival stress tracks.

The large-$|S|$ experiments are CPU-bound rather than state-tensor-memory-bound: the sampling-based meaningful-scale rows cache trajectories over the $13$-history ($|O|{=}3$) or $111$-history ($|O|{=}10$) trees listed in Table~\ref{tab:meaningful-scale}, and do not materialize a dense $|S| \times |S|$ object. For the largest reported $m{=}2$ scaling row (Network Monitoring $n{=}9$, $|S|{=}512$, $|O|{=}3$, $T{=}2$), the dominant pair-by-FSC tensor has ${13 \choose 2}\times 6{,}410 = 499{,}980$ scalar distances, so the operational memory footprint is controlled by history and controller counts rather than raw state count. On the reported Apple M3 Pro / 36\,GB RAM machine, these runs fit comfortably in memory; runtime, not memory pressure, is the practical bottleneck.

\subsection{Command-to-Table Mapping}

\paragraph{Theory-first exact tables.}
Running \texttt{python paper/generate\_theory\_first\_tables.py} writes machine-readable artifacts under \texttt{paper/generated/data/} and LaTeX tables under \texttt{paper/generated/}. Specifically, \texttt{probe\_family\_comparison.csv} and \texttt{table\_probe\_family\_comparison.tex} feed Table~\ref{tab:probe-family-comparison}; \texttt{clock\_aware\_observation\_planning.csv} and \texttt{table\_observation\_planning.tex} feed Table~\ref{tab:observation-planning}; and \texttt{clock\_aware\_latent\_planning.csv} together with \texttt{table\_latent\_planning.tex} feed Table~\ref{tab:latent-planning-exact}.

\paragraph{Operational quick suite.}
Running \texttt{python -m experiments.run\_basic\_results --profile quick --seed 7 --no-parallel --output-dir <DIR>} writes the operational CSV artifacts used for the main empirical tables. The files \texttt{computational\_profile.csv}, \texttt{sampling\_variance.csv}, and \texttt{bootstrap\_coverage.csv} populate Appendix Tables~\ref{tab:computational-profile}, \ref{tab:sampling-variance}, and~\ref{tab:bootstrap-coverage}. The same quick-suite run emits the operational benchmark artifacts from which Tables~\ref{tab:partition-agreement}, \ref{tab:medium-scale}, \ref{tab:tractability}, \ref{tab:pbvi-comparison}, \ref{tab:meaningful-scale}, and~\ref{tab:m2-scaling} are assembled. For the paper's reproduction contract, however, only the subset through Table~\ref{tab:medium-scale} belongs to the fast core-verification tier; the heavier rows are treated as archival stress tracks even though they are produced by the same driver.

\paragraph{Solver settings.}
All $\cW_1$ distances in the main pipeline are solved by \texttt{scipy.optimize.linprog(method="highs")} on normalized probability vectors with exact equality constraints for the transport marginals and nonnegative transport variables; no entropic regularization, Sinkhorn approximation, or early stopping is used. The bisimulation baseline in Appendix Table~\ref{tab:baseline-full} uses the same HiGHS linear-program backend for state-level Wasserstein subproblems, together with $10$ fixed-point iterations and discount $0.9$ before complete-linkage clustering. PBVI uses the finite-horizon implementation in \texttt{experiments.analysis.pbvi\_solve} with horizon equal to the benchmark horizon, $50$ belief points sampled by random forward simulation from $b_0$, discount $\gamma{=}1$, and seed $42$; the quotient-planning comparison applies these identical settings to the original and quotient POMDPs so that the reported speedups isolate compression rather than planner retuning.

\paragraph{Heavy operational tracks.}
The same quick-suite command writes \texttt{meaningful\_scale.csv} and \texttt{m2\_medium\_scale.csv}, which populate Tables~\ref{tab:meaningful-scale} and~\ref{tab:m2-scaling}. These rows are reported separately in the runtime discussion because they dominate the wall-clock budget of the operational pipeline. For inspection without rerunning the heavy jobs, the repository includes \texttt{artifacts/tier3/meaningful\_scale.csv}, \texttt{artifacts/tier3/m2\_medium\_scale.csv}, \texttt{artifacts/tier3/verify\_tier3\_artifacts.py}, and a short \texttt{README.md} explaining their provenance and table mapping.

\paragraph{Long-horizon package.}
Running \texttt{python -m experiments.run\_basic\_results --profile quick --include-hierarchical-scaling --max-horizon 10 --seed 7 --no-parallel --output-dir <DIR>} adds \texttt{hierarchical\_t\_scaling.csv}, \texttt{principal\_fsc\_horizon\_scaling.csv}, and the figures \texttt{fig\_runtime\_vs\_horizon\_log.png} and \texttt{fig\_layered\_bound\_vs\_empirical.png}. These populate Tables~\ref{tab:hierarchical-scaling} and~\ref{tab:principal-horizon}, and Figures~\ref{fig:runtime-horizon} and~\ref{fig:layered-checks}.

\paragraph{Reproduction tiers.}
\begin{table}[H]
\centering
\caption{Three-tier reproduction contract. Tier~I is the intended fast verification target for the paper's core claims; Tiers~II--III are optional extensions for stress-testing the operational pipeline.}
\label{tab:reproduction-tiers}
\footnotesize
\renewcommand{\arraystretch}{1.3}
\setlength{\tabcolsep}{3.5pt}
\begin{tabularx}{\linewidth}{@{} >{\raggedright\arraybackslash}p{1.15cm} >{\raggedright\arraybackslash}X >{\raggedright\arraybackslash}X >{\raggedright\arraybackslash}X @{}}
\toprule
\textbf{Tier} & \textbf{Verification target} & \textbf{Command / artifact path} & \textbf{Serial budget and role} \\
\midrule
I & Theorem tables plus operational core through Table~\ref{tab:medium-scale} & \path{python paper/generate_theory_first_tables.py} and \path{python -m experiments.run_basic_results --profile quick --seed 7 --no-parallel} & ${<}3$\,min; intended reviewer verification target. \\
II & Long-horizon extension (Tables~\ref{tab:hierarchical-scaling}, \ref{tab:principal-horizon}, Figures~\ref{fig:runtime-horizon}, \ref{fig:layered-checks}) & Add \path{--include-hierarchical-scaling --max-horizon 10} to the operational driver & ${\sim}10$\,min serial; horizon-scaling narrative only. \\
III & Heaviest stress rows (Tables~\ref{tab:meaningful-scale}, \ref{tab:m2-scaling}) & Same driver for rerunning; artifact-only: \path{python artifacts/tier3/verify_tier3_artifacts.py} & $5$--$30$\,s per row, up to 2\,h for largest $m{=}2$; archival. Verifier: ${<}1$\,s. \\
\bottomrule
\end{tabularx}
\renewcommand{\arraystretch}{1.0}
\end{table}

\subsection{Archival Operational Stress Tracks}
\label{sec:long-horizon-scaling}

This subsection collects the heavier operational rows that are useful for understanding the computational behaviour of $\Qop$ but are intentionally not part of the paper's core validation target. They remain relevant as appendix-only archival evidence because they stress the same pipeline under longer horizons, larger state spaces, or richer probe families.

\paragraph{Long-horizon scaling via layering and calibrated subsets.}
To address horizon scalability directly, we evaluate the compositional construction of Theorem~\ref{thm:layered-composition} on Tiger with long horizons. For $m{=}1$, we compare monolithic computation against layered segments ($\tau{=}4$ for $T\le 8$, $\tau{=}5$ for $T{=}10$). For $m{=}2$, we calibrate a controller subset at $T{=}6$ and reuse it at larger horizons as an empirical calibrated-subset heuristic without a full $\delta_S$ certificate.

\begin{table}[H]
\centering
\caption{Operational direct vs.\ layered runtime on Tiger full-actions ($m{=}1$, $\varepsilon{=}0.5$, deterministic-stationary probe family).}
\label{tab:hierarchical-scaling}
\small
\begin{tabular}{@{}cccccc@{}}
\toprule
$T$ & Method & Segment $\tau$ & Layers $L$ & Runtime (s) & Histories processed \\
\midrule
4  & Direct   & -- & 1 & 0.05  & 31 \\
4  & Layered  & 4  & 1 & 0.05  & 31 \\
6  & Direct   & -- & 1 & 0.96  & 127 \\
6  & Layered  & 4  & 2 & 0.05  & 38 \\
8  & Direct   & -- & 1 & 16.42 & 511 \\
8  & Layered  & 4  & 2 & 0.10  & 62 \\
10 & Direct   & -- & 1 & 287.36 & 2,047 \\
10 & Layered  & 5  & 2 & 0.44  & 126 \\
\bottomrule
\end{tabular}
\end{table}

\begin{table}[H]
\centering
\caption{Operational exact-for-family and empirical calibrated-subset horizon scaling on Tiger full-actions ($m{=}2$, $\varepsilon{=}0.5$). The calibrated-subset rows reuse a $k{=}10$ subset chosen at $T{=}6$ and do not carry a full $\delta_S$ certificate.}
\label{tab:principal-horizon}
\small
\begin{tabular}{@{}cccccc@{}}
\toprule
$T$ & Method & FSCs used & Runtime (s) & Classes & ARI vs.\ full \\
\midrule
4 & Full op.-exact & 147 & 2.53 & 10 & 1.000 \\
5 & Full op.-exact & 147 & 10.95 & 16 & 1.000 \\
6 & Full op.-exact & 147 & 46.11 & 30 & 1.000 \\
7 & Full op.-exact & 147 & 194.27 & 50 & 1.000 \\
\midrule
7 & Calibrated subset (empirical) & 10 & 12.96 & 50 & 0.9997 \\
8 & Calibrated subset (empirical) & 10 & 53.69 & 77 & -- \\
\bottomrule
\end{tabular}
\end{table}

The layered curve in Table~\ref{tab:hierarchical-scaling} demonstrates the intended complexity shift: monolithic runtime grows with the full history tree, while layered runtime tracks short segments. At $T{=}10$, layering reduces runtime by roughly three orders of magnitude while preserving the operational error-accumulation narrative from Theorem~\ref{thm:layered-composition}. For $m{=}2$, Table~\ref{tab:principal-horizon} shows that a calibrated subset ($k{=}10$, chosen from a $T{=}6$ calibration run) preserves near-operational-exact partition structure at $T{=}7$ (ARI $=0.9997$) and enables extension to $T{=}8$ without full FSC enumeration. These long-horizon calibrated-subset runs are empirical scalability results: unlike Table~\ref{tab:partition-agreement}, we do not report a full probe-gap certificate $\delta_S$ for them, and we do not present them as theorem-level guarantees.

\begin{figure}[htbp]
\centering
\includegraphics[width=0.9\linewidth]{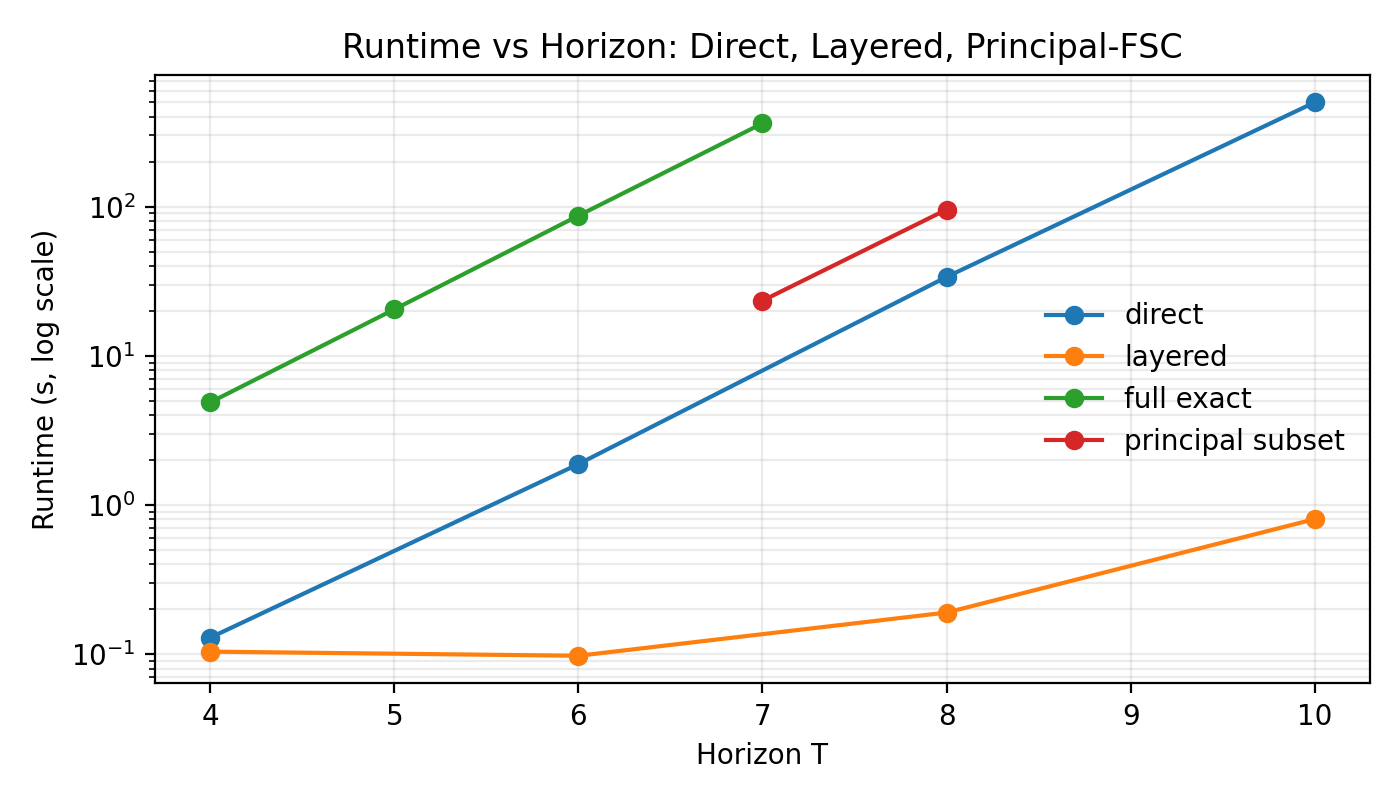}
\caption{Operational runtime scaling with horizon (log scale): monolithic, layered, and empirical calibrated-subset tracks.}
\label{fig:runtime-horizon}
\end{figure}

\begin{figure}[htbp]
\centering
\includegraphics[width=0.9\linewidth]{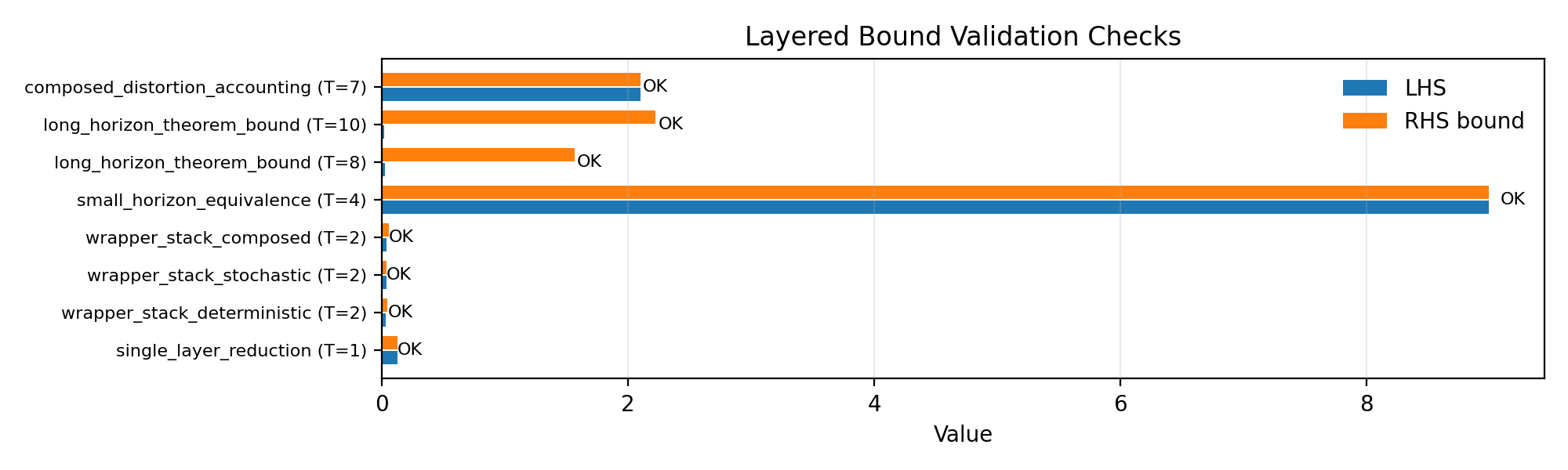}
\caption{Layered bound validation (Theorem~\ref{thm:layered-composition}). Each bar compares the empirical LHS distortion (blue) against the theoretical RHS bound (orange) for eight test cases spanning single-layer reduction, wrapper contraction (deterministic and stochastic), and long-horizon theorem bounds ($T \in \{4, 8, 10\}$). All checks satisfy LHS~$\leq$~RHS.}
\label{fig:layered-checks}
\end{figure}

\paragraph{Tractability summary.}
Table~\ref{tab:tractability} summarizes how the three scalability dimensions are addressed. The history dimension $|O|^T$ is addressed by segmenting horizon into short layers (Theorem~\ref{thm:layered-composition}, Proposition~\ref{prop:layered-complexity}), demonstrated up to $T{=}10$. The FSC dimension is reduced by greedy selection: on the operational exact-for-family benchmarks, $k{=}5$ recovers the full probe envelope after \emph{a posteriori} certification ($\delta_S{=}0$), while on long-horizon Tiger a calibrated $k{=}10$ subset provides a near-operational-exact empirical surrogate without a full probe-gap certificate. The state dimension is decoupled by sampling ($0.30$\,s at $|S|{=}100$, Table~\ref{tab:medium-scale}). Core benchmark suites remain lightweight and are the intended serial reproduction target; the stress tracks in this appendix are optional single-core extensions.

\begin{table}[H]
\centering
\caption{Scalability: each complexity dimension and how it is addressed, separating operational exact-for-family benchmark certification from empirical long-horizon calibrated subsets.}
\label{tab:tractability}
\scriptsize
\begin{tabular}{@{}lllp{0.38\linewidth}@{}}
\toprule
Dimension & Growth & Technique & Demonstrated \\
\midrule
States $|S|$ & $O(|S|^2)$ per belief & Sampling-based $\cW_1$ & $|S|{=}4{,}096$ (sampling), $|S|{=}100$ (sampling) \\
FSCs $|\cF|$ & $|A|^m m^{m|O|}$ & Greedy FSC subsets & \shortstack[l]{Certified: $6{,}405 \to 5$, $147 \to 5$\\ $m{=}2$ scaling: $5{,}193$ FSCs on RockSample\\ Empirical: $147 \to 10$ calibrated subset} \\
Histories $|O|^T$ & Exp.\ in $T$ & Layered horizon decomposition & Direct to $T{=}10$; layered to $T{=}20$ \\
Observations $|O|$ & $|O|^T$ histories & $\delta_O$-coarsening (Prop.~\ref{prop:obs-resolution}) & GridWorld $4 \to 2$ obs \\
\bottomrule
\end{tabular}
\end{table}

\paragraph{Large-state and high-memory stress rows.}
Table~\ref{tab:meaningful-scale} records appendix-level archival rows at scales exceeding $|S|{=}1000$ or effective horizon $T{=}20$. Sampling-based $\cW_1$ estimation decouples runtime from $|S|$, enabling direct computation on Network Monitoring with $|S|{=}4{,}096$ states. The layered construction (Theorem~\ref{thm:layered-composition}) extends the effective horizon to $T{=}20$ by composing five $\tau{=}4$ segments, each computed independently via sampling. The observation alphabet $|O|$ remains the true computational bottleneck: at $|O|{=}10$, the history tree has $111$ nodes at depth $2$ versus $13$ at $|O|{=}3$. Bootstrap CI widths at this scale are ${\leq}0.17$ (Table~\ref{tab:meaningful-scale}), comparable to medium-scale benchmarks (Table~\ref{tab:medium-scale}), confirming that $500$ trajectories provide reliable distance estimates even at $|S|{=}4{,}096$.

\begin{table}[H]
\centering
\caption{Meaningful-scale experiments: sampling-based and layered quotient construction ($m{=}1$, $500$ trajectories). The 95\% CI column reports the bootstrap confidence interval on $\max_{h,h'} \cW_1$ ($1{,}000$ resamples); for the layered track, the CI is on the segment-level ($\tau{=}4$) cache.}
\label{tab:meaningful-scale}
\small
\resizebox{\linewidth}{!}{\begin{tabular}{@{}llcccccccc@{}}
\toprule
Track & Benchmark & $|S|$ & $|O|$ & $T$ & Method & $\varepsilon{=}0$ & $\varepsilon{=}0.3$ & Runtime (s) & 95\% CI on $\max \cW_1$ \\
\midrule
Scale-up & Net.\ Mon.\ ($n{=}12$) & 4{,}096 & 3 & 2 & Sampling & 13 & 5 & $\sim 5$ & $[0.31,\, 0.43]$ \\
Large $|O|$ & Random & 2{,}000 & 10 & 2 & Sampling & 111 & 10 & $\sim 30$ & $[0.10,\, 0.20]$ \\
Long $T$ & Net.\ Mon.\ ($n{=}10$) & 1{,}024 & 3 & 20 & Layered & 13 & 5 & $\sim 25$ & $[0.67,\, 0.84]$ \\
\bottomrule
\end{tabular}}
\end{table}

\paragraph{Higher-memory probe-family scaling.}
Proposition~\ref{prop:probe-class-coarsening} predicts that a richer probe family produces a finer partition. Table~\ref{tab:m2-scaling} illustrates this beyond the operational exact-for-family benchmarks of Table~\ref{tab:partition-agreement}, using sampling-based $\cW_1$ estimation. On RockSample$(4,4)$ ($|S|{=}257$, $9$ actions, $3$ observations), the $m{=}2$ probe family ($5{,}193$ FSCs) maintains $5$ classes at $\varepsilon{=}0.5$, while the $m{=}1$ family ($9$ FSCs) collapses to $3$. On Network Monitoring ($n{=}4$, $|S|{=}16$) at $T{=}3$, $m{=}2$ ($1{,}605$ FSCs) sustains $14$ classes through $\varepsilon{=}0.3$, while $m{=}1$ ($5$ FSCs) drops to $6$. On Network Monitoring ($n{=}9$, $|S|{=}512$), $m{=}2$ ($6{,}410$ FSCs) preserves $5$ classes at $\varepsilon{=}0.3$ while $m{=}1$ ($10$ FSCs) drops to $4$. In all three benchmarks, the max pairwise $\cW_1$ increases under $m{=}2$, confirming that additional controller memory exposes finer-grained behavioural differences (Proposition~\ref{prop:probe-class-coarsening}). These rows are reported as archival stress tests rather than as part of the core theorem-validation suite.

\begin{table}[H]
\centering
\caption{Probe-family scaling: $m{=}1$ vs.\ $m{=}2$ quotient class counts at medium scale (sampling-based, $500$ trajectories). The 95\% CI column reports the bootstrap confidence interval on $\max_{h,h'} \cW_1$ ($1{,}000$ resamples).}
\label{tab:m2-scaling}
\scriptsize
\resizebox{\linewidth}{!}{\begin{tabular}{@{}llccccccccc@{}}
\toprule
Benchmark & $|S|$ & $m$ & $T$ & FSCs & $\varepsilon{=}0$ & $\varepsilon{=}0.1$ & $\varepsilon{=}0.3$ & $\varepsilon{=}0.5$ & Runtime (s) & 95\% CI on $\max \cW_1$ \\
\midrule
RockSample$(4,4)$ & 257 & 1 & 2 & 9 & 5 & 5 & 4 & 3 & 0.5 & $[0.35,\, 0.47]$ \\
RockSample$(4,4)$ & 257 & 2 & 2 & 5{,}193 & 5 & 5 & 5 & 5 & 321 & $[1.00,\, 1.00]$ \\
\midrule
Net.\ Mon.\ ($n{=}4$) & 16 & 1 & 3 & 5 & 14 & 9 & 6 & 5 & 1.1 & $[0.48,\, 0.66]$ \\
Net.\ Mon.\ ($n{=}4$) & 16 & 2 & 3 & 1{,}605 & 14 & 14 & 14 & 6 & 353 & $[0.67,\, 0.84]$ \\
\midrule
Net.\ Mon.\ ($n{=}9$) & 512 & 1 & 2 & 10 & 5 & 5 & 4 & 3 & 11 & $[0.28,\, 0.40]$ \\
Net.\ Mon.\ ($n{=}9$) & 512 & 2 & 2 & 6{,}410 & 5 & 5 & 5 & 3 & 6{,}802 & $[0.39,\, 0.49]$ \\
\bottomrule
\end{tabular}}
\end{table}

\subsection{Tiger Reproduction and GridWorld Capacity Sweep}

With $m{=}1$, $T{=}2$ on the listen-only Tiger POMDP, the experiment reproduces $\cW_1(P(\cdot \mid h_L), P(\cdot \mid h_R)) = 0.49$, $7$ operational exact classes at $\varepsilon = 0$, and $3$ classes at $\varepsilon = 0.5$.

The GridWorld capacity sweep below shows quotient class counts across memory bounds and thresholds.

\begin{table}[H]
\centering
\caption{Quotient class counts for GridWorld $3 \times 3$ ($|S|=9$, $T=2$). Total histories: $21$.}
\label{tab:capacity-sweep-gridworld}
\small
\begin{tabular}{@{}ccccc@{}}
\toprule
$m$ & $\varepsilon$ & Classes & Compression & Policies \\
\midrule
1 & 0.0 & 6 & 0.29 & 5 \\
1 & 0.2 & 5 & 0.24 & 5 \\
1 & 0.35 & 3 & 0.14 & 5 \\
1 & 0.5 & 3 & 0.14 & 5 \\
\midrule
2 & 0.0 & 6 & 0.29 & 6{,}405 \\
2 & 0.2 & 6 & 0.29 & 6{,}405 \\
2 & 0.35 & 6 & 0.29 & 6{,}405 \\
2 & 0.5 & 3 & 0.14 & 6{,}405 \\
\bottomrule
\end{tabular}
\end{table}

\subsection{Value-Function Error Bounds (Full)}

\begin{table}[H]
\centering
\caption{Value-function error bounds across reward scales. Panel~A: Tiger full actions with synthetic $L_R{=}1$ reward (tight bounds). Panel~B: Tiger with standard reward $L_R{=}110$ (vacuous---bound exceeds reward range). Panel~C: GridWorld $3{\times}3$ goal reward $L_R \approx 2$ (moderate---informative bound at small $\varepsilon$). All use $m{=}1$, $T{=}2$. The column $L_R T \varepsilon$ is the provable bound from Theorem~\ref{thm:value-bound}; the column $L_R T D$ reports the tighter empirical proxy $L_R \cdot T \cdot D_{m,T}^{\cW}$, which is always valid but follows from the partition construction only when $D_{m,T}^{\cW} \leq \varepsilon$ (observed in all tested configurations at $\varepsilon > 0$).}
\label{tab:value-bounds}
\small
\begin{tabular}{@{}cccccc@{}}
\toprule
$\varepsilon$ & Empirical error & $D_{m,T}^{\cW}$ & $L_R T \varepsilon$ & $L_R T D$ & Prop~\ref{prop:canonical-quotient} \\
\midrule
\multicolumn{6}{c}{\emph{Panel A: Tiger, synthetic $L_R{=}1$}} \\
\midrule
0.0 & $\approx 0$ & 0.00 & 0.00 & 0.00 & 0.0 \\
0.1 & $\approx 0$ & 0.00 & 0.20 & 0.00 & 3.4 \\
0.3 & $\approx 0$ & 0.00 & 0.60 & 0.00 & 10.2 \\
0.5 & 0.12 & 0.24 & 1.00 & 0.49 & 17.0 \\
\midrule
\multicolumn{6}{c}{\emph{Panel B: Tiger, standard reward $L_R{=}110$ (vacuous)}} \\
\midrule
0.0 & 0.00 & 0.00 & 0.00 & 0.00 & 0.0 \\
0.1 & 0.00 & 0.00 & 22.0 & 0.00 & 374.0 \\
0.3 & 0.00 & 0.00 & 66.0 & 0.00 & 1122.0 \\
0.5 & 2.89 & 0.24 & 110.0 & 53.33 & 1870.0 \\
\midrule
\multicolumn{6}{c}{\emph{Panel C: GridWorld $3{\times}3$, goal reward $L_R \approx 2$}} \\
\midrule
0.0 & 0.00 & 0.00 & 0.00 & 0.00 & 0.0 \\
0.1 & $\approx 0$ & $\approx 0$ & $\approx 0.4$ & $\approx 0$ & 5.8 \\
0.3 & $\approx 0$ & $\approx 0$ & $\approx 1.2$ & $\approx 0$ & 17.3 \\
0.5 & $\approx 0$ & $\approx 0$ & $\approx 2.0$ & $\approx 0$ & 28.8 \\
\bottomrule
\end{tabular}
\end{table}

\subsection{Horizon Gap Analysis}

As $T$ increases from $2$ to $10$ on Tiger ($\varepsilon = 0.5$), empirical value error remains small while the pseudometric bound stays informative.
The canonical worst-case aggregation bound grows from $17$ (at $T{=}2$) to $405$ (at $T{=}10$), confirming its conservative worst-case nature.
The non-monotone (sawtooth) pattern arises because the set of merged history classes changes discretely at each $\varepsilon$ threshold, so per-class distortions can drop when a coarser partition happens to align better with the reward structure.

\subsection{Operational Probe Family Sanity Check}

Across all $147$ deterministic FSCs ($m{=}2$, Tiger), the maximum $\cW_1$ between $h_L, h_R$ is $0.49$.
Among $200$ sampled stochastic FSCs ($40$ per seed, seeds $\{7, 42, 123, 256, 999\}$), the per-seed maximum $\cW_1$ is $0.38 \pm 0.07$ (mean $\pm$ std, range $0.31$--$0.47$)---no sampled stochastic stationary controller exceeds the deterministic maximum of $0.49$ on this benchmark.
This is a benchmark-specific sanity check for the operational stationary probe family only: Proposition~\ref{prop:stationary-counterexample} gives a separate clock-aware exact POMDP where stochastic stationary FSCs strictly outperform deterministic stationary ones.

\subsection{Observation Noise Sensitivity}

\begin{table}[H]
\centering
\caption{Noise sensitivity (Tiger, $m=1$, $T=2$). Range over accuracy $\in \{0.70, 0.75, 0.80, 0.85, 0.90, 0.95\}$.}
\small
\begin{tabular}{@{}ccccc@{}}
\toprule
$\varepsilon$ & Min cl. & Max cl. & Min val.\ err. & Max val.\ err. \\
\midrule
0.0 & 7 & 7 & 0.000 & 0.000 \\
0.2 & 5 & 7 & 0.000 & 0.000 \\
0.4 & 3 & 7 & 0.000 & 0.108 \\
0.6 & 3 & 3 & 0.058 & 0.145 \\
\bottomrule
\end{tabular}
\end{table}

\subsection{Full Baseline Comparison}

\begin{table}[H]
\centering
\caption{Full baseline comparison ($m=1$, $T=2$), both benchmarks.}
\label{tab:baseline-full}
\small
\begin{tabular}{@{}llccccc@{}}
\toprule
Benchmark & Method & $\varepsilon$ & Classes & Val.\ err. & $D_{m,T}^{\cW}$ & Time (s) \\
\midrule
\multicolumn{7}{c}{\emph{Tiger full actions}} \\
\midrule
& $\varepsilon$-quotient & 0.3 & 4 & 0.000 & 0.000 & $\ast$ \\
& Truncation ($d{=}1$) & 0.3 & 5 & 0.000 & 0.000 & $<0.01$ \\
& Random partition & 0.3 & 3.7 & 0.041 & 0.082 & $<0.01$ \\
& Belief-distance & 0.3 & 6 & 0.000 & 0.000 & $<0.01$ \\
& Bisimulation & 0.3 & 6 & 0.000 & 0.000 & 0.05 \\
\midrule
& $\varepsilon$-quotient & 0.5 & 3 & 0.122 & 0.245 & $\ast$ \\
& Truncation ($d{=}1$) & 0.5 & 5 & 0.000 & 0.000 & $<0.01$ \\
& Random partition & 0.5 & 3.0 & 0.122 & 0.245 & $<0.01$ \\
& Belief-distance & 0.5 & 5 & 0.000 & 0.000 & $<0.01$ \\
& Bisimulation & 0.5 & 5 & 0.000 & 0.000 & 0.05 \\
\midrule
\multicolumn{7}{c}{\emph{GridWorld $3 \times 3$}} \\
\midrule
& $\varepsilon$-quotient & 0.3 & 4 & 0.099 & 0.092 & $\ast$ \\
& Truncation ($d{=}1$) & 0.3 & 9 & 0.000 & 0.000 & $<0.01$ \\
& Random partition & 0.3 & 3.7 & 0.077 & 0.163 & $<0.01$ \\
& Belief-distance & 0.3 & 9 & 0.000 & 0.052 & 0.02 \\
& Bisimulation & 0.3 & 3 & 0.104 & 0.166 & 0.10 \\
\midrule
& $\varepsilon$-quotient & 0.5 & 3 & 0.104 & 0.166 & $\ast$ \\
& Truncation ($d{=}1$) & 0.5 & 9 & 0.000 & 0.000 & $<0.01$ \\
& Random partition & 0.5 & 3.0 & 0.104 & 0.166 & $<0.01$ \\
& Belief-distance & 0.5 & 6 & 0.099 & 0.111 & 0.02 \\
& Bisimulation & 0.5 & 3 & 0.104 & 0.166 & 0.10 \\
\bottomrule
\multicolumn{7}{l}{\footnotesize $\ast$ Cache construction dominated; per-$\varepsilon$ partition is $<0.01$\,s.}
\end{tabular}
\end{table}

\subsection{Ablation Studies}

\begin{table}[H]
\centering
\caption{Ablation studies (Tiger full actions, $T=2$).}
\small
\begin{tabular}{@{}llcc@{}}
\toprule
Ablation & Variant & Classes ($\varepsilon{=}0.3$) & Note \\
\midrule
Metric & $\cW_1$ (discrete $d_O$) & 7 & Default \\
        & TV-equivalent & 7 & Same on Tiger (binary $O$) \\
\midrule
Memory & $m=1$ (3 FSCs) & 7 & Coarser \\
       & $m=2$ (147 FSCs) & 14 & Finer, 0 value error \\
\midrule
Approx. & Full enum.\ (147 FSCs) & 14 & Exact \\
         & Greedy $k{=}5$ & 14 & ARI $= 1.0$, $2.1\times$ faster \\
\bottomrule
\end{tabular}
\end{table}

\subsection{Low-Rank Analysis}

\begin{table}[H]
\centering
\caption{Effective rank of the FSC distinguishing matrix.}
\small
\begin{tabular}{@{}llccccc@{}}
\toprule
Benchmark & $m$ & Depth & FSCs & Rank (90\%) & Rank (95\%) & Rank (99\%) \\
\midrule
Tiger ($T{=}4$) & 1 & 3 & 3 & 1 & 1 & 1 \\
Tiger ($T{=}4$) & 2 & 3 & 147 & 3 & 4 & 7 \\
Grid $3{\times}3$ ($T{=}2$) & 1 & 1 & 5 & 1 & 2 & 4 \\
Grid $3{\times}3$ ($T{=}2$) & 2 & 1 & 6{,}405 & 2 & 3 & 5 \\
\bottomrule
\end{tabular}
\end{table}

\subsection{Larger-Scale Results}

\begin{table}[H]
\centering
\caption{Quotient class counts for larger POMDPs ($m=1$). Total histories: $21$ ($T{=}2$), $85$ ($T{=}3$).}
\small
\begin{tabular}{@{}lccccccc@{}}
\toprule
Benchmark & $T$ & $\varepsilon{=}0$ & $0.1$ & $0.2$ & $0.3$ & $0.5$ & $0.6$ \\
\midrule
Grid $5{\times}5$ & 2 & 6 & 6 & 5 & 4 & 3 & 3 \\
Grid $5{\times}5$ & 3 & 22 & 16 & 12 & 9 & 6 & 5 \\
Random $|S|{=}20$ & 2 & 6 & 3 & 3 & 3 & 3 & 3 \\
Random $|S|{=}20$ & 3 & 22 & 4 & 4 & 4 & 4 & 4 \\
\bottomrule
\end{tabular}
\end{table}

\subsection{Scaling and Timing}

\begin{table}[H]
\centering
\caption{Scaling: wall-clock time vs.\ state-space size (GridWorld, $T{=}2$, $m{=}1$, $\varepsilon{=}0.3$).}
\small
\begin{tabular}{@{}ccccc@{}}
\toprule
Grid & $|S|$ & Classes & Compression & Time (s) \\
\midrule
$3 \times 3$ & 9 & 4 & $0.19$ & 0.03 \\
$4 \times 4$ & 16 & 4 & $0.19$ & 0.04 \\
$5 \times 5$ & 25 & 4 & $0.19$ & 0.05 \\
$6 \times 6$ & 36 & 4 & $0.19$ & 0.07 \\
\bottomrule
\end{tabular}
\end{table}

\begin{table}[H]
\centering
\caption{Wall-clock time per experiment (quick profile, single CPU core).}
\small
\begin{tabular}{@{}lc@{}}
\toprule
Experiment & Time (s) \\
\midrule
Tiger reproduction & $<0.1$ \\
Capacity sweep & $0.5$--$1.0$ \\
Value loss bounds & $0.3$--$0.5$ \\
Metric sensitivity & $1.0$--$2.0$ \\
Baseline comparison & $0.5$--$1.0$ \\
Low-rank analysis & $0.5$--$2.0$ \\
\midrule
\textbf{Total} & $\mathbf{<20}$ \\
\bottomrule
\end{tabular}
\end{table}

\subsection{Computational Bottleneck Breakdown}

\begin{table}[H]
\centering
\caption{Pipeline stage breakdown (single CPU core, $m{=}1$, $T{=}2$). Distance cache ($\cW_1$ LP solves) dominates in all non-trivial cases.}
\label{tab:computational-profile}
\small
\begin{tabular}{@{}lcccccc@{}}
\toprule
Benchmark & $|S|$ & $m$ & FSCs & Enum (s) & Cache (s) & Cache \% \\
\midrule
Tiger & 2 & 1 & 3 & ${<}0.01$ & $0.004$ & $99.0$ \\
GridWorld $3{\times}3$ & 9 & 1 & 5 & ${<}0.01$ & $0.034$ & $99.8$ \\
RockSample(4,4) & 257 & 1 & 9 & ${<}0.01$ & $0.472$ & $100.0$ \\
\bottomrule
\end{tabular}
\end{table}

\subsection{Sampling Convergence and Variance}

\begin{table}[H]
\centering
\caption{Operational sampling convergence on GridWorld $3{\times}3$ ($m{=}1$, $T{=}2$, $5$ replications per sample count, averaged over $\varepsilon \in \{0, 0.3, 0.5\}$). Minimum ARI stabilizes by $100$ trajectories.}
\label{tab:convergence}
\small
\begin{tabular}{@{}cccc@{}}
\toprule
Trajectories & Mean ARI & Min ARI & Replications \\
\midrule
50 & 0.991 & 0.961 & 5 \\
100 & 0.994 & 0.971 & 5 \\
250 & 0.994 & 0.971 & 5 \\
500 & 0.998 & 0.971 & 5 \\
1{,}000 & 0.994 & 0.971 & 5 \\
\bottomrule
\end{tabular}
\end{table}

\begin{table}[H]
\centering
\caption{Partition stability: class counts across $10$ independent random seeds (GridWorld $10{\times}10$, $|S|{=}100$, $m{=}1$, $T{=}2$, $500$ trajectories).}
\label{tab:sampling-variance}
\small
\begin{tabular}{@{}ccccc@{}}
\toprule
$\varepsilon$ & Mean classes & Std & Min & Max \\
\midrule
0.0 & 6.0 & 0.0 & 6 & 6 \\
0.1 & 6.0 & 0.0 & 6 & 6 \\
0.25 & 4.0 & 0.0 & 4 & 4 \\
0.45 & 3.0 & 0.0 & 3 & 3 \\
\bottomrule
\end{tabular}
\end{table}

The zero variance at $500$ trajectories indicates that the sampling noise in the estimated $\cW_1$ distances is well below the complete-linkage clustering threshold at all tested $\varepsilon$ values. Even at $50$ trajectories, mean ARI exceeds $0.99$, and the minimum ARI stabilizes at $0.971$ by $100$ trajectories, confirming that the partition is robust to sampling noise.

\begin{remark}[Sample complexity of empirical $\cW_1$]
\label{rem:sample-complexity}
On a finite support of size $K$, the empirical $\cW_1$ distance converges at rate $O(n^{-1/2})$ in expectation~\citep{fournier2015rate}, where $n$ is the number of samples per distribution and $K = |O|^{T-t}$ is the number of possible future observation sequences at depth $t$.
For the bounded-agent regimes targeted here ($|O| \leq 7$, $T \leq 4$), $K$ remains modest (at most $|O|^T = 2{,}401$), and $500$ trajectories per (history, policy) pair suffice empirically (Table~\ref{tab:convergence}).
The max-over-policies aggregation introduces a union-bound factor of $|\Pi|$; in practice, the greedy-selected subsets ($|\Pi_S| \leq 10$) keep this overhead negligible.
\end{remark}

\subsection{Bootstrap CI Coverage Validation}

\begin{table}[H]
\centering
\caption{Empirical coverage of 95\% bootstrap CIs on $\max_{h,h'} \cW_1$ (GridWorld $3{\times}3$, $m{=}1$, $T{=}2$, $200$ replications).}
\label{tab:bootstrap-coverage}
\small
\begin{tabular}{@{}ccc@{}}
\toprule
Trajectories & Coverage & Mean CI width \\
\midrule
100 & $97.0\%$  & $0.199$ \\
250 & $96.5\%$  & $0.127$ \\
500 & $96.0\%$  & $0.090$ \\
\bottomrule
\end{tabular}
\end{table}

Coverage is within $2\%$ of the nominal $95\%$ level at all sample sizes, confirming reliable uncertainty quantification for the max-distance statistic.

\subsection{Additional Benchmark Results}

\begin{table}[H]
\centering
\caption{Quotient class counts on additional benchmarks ($m{=}1$, $T{=}2$). Total histories: $13$.}
\label{tab:new-benchmarks}
\small
\begin{tabular}{@{}lcccccc@{}}
\toprule
Benchmark & $|S|$ & $|A|$ & $|O|$ & $\varepsilon{=}0$ & $\varepsilon{=}0.3$ & $\varepsilon{=}0.5$ \\
\midrule
Hallway ($L{=}10$) & 10 & 3 & 3 & 5 & 5 & 3 \\
Hallway ($L{=}20$) & 20 & 3 & 3 & 5 & 5 & 3 \\
Network ($n{=}4$) & 16 & 5 & 3 & 5 & 4 & 3 \\
Network ($n{=}5$) & 32 & 6 & 3 & 5 & 4 & 3 \\
\bottomrule
\end{tabular}
\end{table}

Both benchmarks exhibit the expected monotonicity: class count is non-increasing in $\varepsilon$.
Hallway's chain topology produces identical compression at $L{=}10$ and $L{=}20$ (the cyclically assigned landmarks create equivalent observation structure at both scales).
Network Monitoring shows slightly finer resolution (4 classes at $\varepsilon{=}0.3$) due to the probing actions creating action-dependent observation structure.

\subsection{Downstream Planning Impact}

\begin{table}[H]
\centering
\caption{Planning on original vs.\ quotient: value gap under exhaustive policy search ($T{=}2$).}
\label{tab:planning-speedup}
\small
\begin{tabular}{@{}llccccc@{}}
\toprule
Benchmark & $m$ & $\varepsilon$ & Policies & Classes & $V^*_{\text{orig}}$ & Value gap \\
\midrule
Tiger & 1 & 0.0 & 3 & 4 & 0.373 & 0.000 \\
Tiger & 1 & 0.3 & 3 & 4 & 0.373 & 0.000 \\
Tiger & 1 & 0.5 & 3 & 3 & 0.373 & 0.000 \\
Tiger & 2 & 0.3 & 147 & 4 & 0.373 & 0.000 \\
Tiger & 2 & 0.5 & 147 & 3 & 0.373 & 0.000 \\
\midrule
Hallway & 1 & 0.0 & 3 & 5 & 0.262 & 0.000 \\
Hallway & 1 & 0.3 & 3 & 5 & 0.262 & 0.000 \\
Hallway & 1 & 0.5 & 3 & 3 & 0.262 & 0.082 \\
\bottomrule
\end{tabular}
\end{table}

The value gap (difference between the original-optimal and quotient-optimal policy values, both evaluated on the original POMDP) is zero at $\varepsilon \leq 0.3$ for all settings.
At $\varepsilon{=}0.5$ on Hallway, the quotient selects a different policy with value gap $0.082$, consistent with the bound $L_R T \varepsilon = 1.0$.
The compression from $5$ classes to $3$ at $\varepsilon{=}0.5$ would reduce the effective planning state space by $40\%$ for planners that operate on history equivalence classes.

\subsection{Initial-Belief Sensitivity}

\begin{table}[H]
\centering
\caption{Partition stability under initial-belief perturbation (Tiger full actions, $m{=}1$, $T{=}2$). ARI measured against the uniform-belief ($b_0 = [0.5, 0.5]$) partition.}
\label{tab:belief-sensitivity}
\small
\begin{tabular}{@{}cccc@{}}
\toprule
$b_0(s_L)$ & $\varepsilon{=}0$ (classes / ARI) & $\varepsilon{=}0.3$ (classes / ARI) & $\varepsilon{=}0.5$ (classes / ARI) \\
\midrule
0.5 & 4 / 1.00 & 4 / 1.00 & 3 / 1.00 \\
0.3 & 4 / 1.00 & 4 / 1.00 & 3 / 1.00 \\
0.1 & 4 / 1.00 & 3 / 0.89 & 3 / 1.00 \\
0.7 & 4 / 1.00 & 4 / 1.00 & 3 / 1.00 \\
0.9 & 4 / 1.00 & 3 / 0.89 & 3 / 1.00 \\
\bottomrule
\end{tabular}
\end{table}

The partition is stable under moderate perturbations of $b_0$: ARI $= 1.0$ for $b_0 \in [0.3, 0.7]$ at all $\varepsilon$.
Only extreme beliefs ($b_0(s_L) \in \{0.1, 0.9\}$) at intermediate $\varepsilon{=}0.3$ produce a different partition (3 vs.\ 4 classes, ARI $= 0.89$), consistent with the continuity of belief posteriors in $b_0$.

\section{Notation Reference}
\label{app:notation}

\begin{table}[H]
\centering
\caption{Principal notation.}
\small
\begin{tabular}{@{}ll@{}}
\toprule
\textbf{Symbol} & \textbf{Meaning} \\
\midrule
$M = \langle S,A,O,P,Z,R,b_0\rangle$ & Finite POMDP \\
$S, A, O$ & State, action, observation sets \\
$P(s,a,s')$ & Transition kernel $\bP(s' \mid s, a)$ \\
$Z(s',a,o)$ & Observation kernel $\bP(o \mid s', a)$ \\
$R(s,a)$ & Reward function \\
$b_t \in \Delta(S)$ & Belief state at time $t$ \\
$\pi = \langle N, \alpha, \beta, n_0\rangle$ & Stochastic FSC \\
$m$ & Memory bound (max FSC nodes) \\
$T$ & Finite horizon \\
$\Pi_{m,T}$ & Set of stochastic FSCs with $\leq m$ nodes \\
$d_O$ & Ground metric on observations \\
$\delta_O$ & Observation resolution threshold \\
$|O_\delta|$ & $\delta_O$-covering number of observation space \\
$Q_{m,T,\delta}$ & Quotient POMDP under $(m,T,\delta_O)$-bounded agents \\
$\cW_1$ & 1-Wasserstein distance \\
$D_{m,T}^{\cW}(M,N)$ & Closed-loop Wasserstein pseudometric \\
$h = (o_1, \ldots, o_t)$ & Observation history \\
$\equiv_{m,T}$ & Exact bounded-interaction equivalence \\
$[h]$, $[h]_\varepsilon$ & Equivalence class (exact-for-family / $\varepsilon$-approximate) \\
$Q_{m,T}(M)$ & Quotient POMDP \\
$\bar{b}_{[h]}$ & Aggregated belief for class $[h]$ \\
$V_M^\pi$ & Value of policy $\pi$ in POMDP $M$ \\
\bottomrule
\end{tabular}
\end{table}

\end{document}